%% file: SGDInitialRegularization.tex
\newcommand{\titt}{Stochastic gradient descent with initial regularization}

\newcommand{\commentt}[2]{#1}
\newcommand{\comment}[1]{}

\newcommand{\var}{{\rm Var}}
\newcommand{\tr}{{\rm tr}}

\commentt{
\newcommand{\citet}{\citeasnoun}
\documentclass[a4paper,11pt]{article}
\usepackage{graphicx,amsthm}
\usepackage[full]{harvard}

\usepackage[ps,dvips]{xy}
\title{\titt}
\author{Nabil Kahal\'e
\thanks{\emph{ESCP Business School, Paris, France; {e-mail: }{nkahale@escp.eu}.}}}
\date{}
\usepackage{amstext}
\usepackage{amssymb}
\usepackage{amsmath}
\setcounter{equation}{0}
\usepackage{setspace}
\usepackage[margin=1in]{geometry}
\onehalfspacing
\usepackage[english]{babel}
\bibliographystyle{dcu}
\usepackage{color}
\usepackage{textcomp}
\usepackage{varioref}
\usepackage{algpseudocode}
\usepackage{algorithm}

\usepackage{xspace}
\usepackage{pgfplots}
\usepackage{caption,subcaption}

\begin{document}

\newtheorem{theorem}{Theorem}
\newtheorem{remark}{Remark}
\newtheorem{lemma}{Lemma}
\newtheorem{proposition}{Proposition}
\newtheorem{corollary}{Corollary}
\maketitle
\newcommand{\citep}{\cite}
}
{
%
%


\RequirePackage{bm}
\RequirePackage{endnotes}

\OneAndAHalfSpacedXI



\usepackage{algorithm}
\usepackage{algpseudocode}
\usepackage{tikz}
\usetikzlibrary{matrix}

\usepackage[sort&compress]{natbib}
 \bibpunct[, ]{[}{]}{,}{n}{}{,}%
 \def\bibfont{\small}%
 \def\bibsep{\smallskipamount}%
 \def\bibhang{24pt}%
 \def\newblock{\ }%
 \def\BIBand{and}%

\EquationsNumberedThrough    

\TheoremsNumberedThrough     
\ECRepeatTheorems  %

\MANUSCRIPTNO{}


\usepackage{pgfplots}
\usepackage{caption,subcaption}
\begin{document}
\RUNAUTHOR{Kahal\'e}
\RUNTITLE{}
\TITLE{\titt}
\ARTICLEAUTHORS{%
\AUTHOR{Nabil Kahal\'e}

\AFF{ESCP Business School, Paris, France, \EMAIL{nkahale@escp.eu} \URL{}}
} 
\KEYWORDS{\keyy}
}
\begin{abstract}
We analyze a variant of stochastic gradient descent with initial regularization (SGDIR) and derive dimension-free upper bounds on its expected excess risk for the squared loss. In the noiseless case, we obtain new bounds  for both averaged and non-averaged SGDIR  under   moment, source, and capacity assumptions. For a particular value of the source parameter, these bounds are of order \(m^{-2}\log^{2}m\), where the number of training samples is of order \(m\). For  another value of the source parameter, we obtain, for any \(\epsilon>0\),   bounds  of order  \(m^{-3+\epsilon}\), provided that the capacity parameter exceeds \(\epsilon^{-1}\). \comment{ In the noiseless case, our bounds for both averaged and non-averaged SGDIR are, in certain regimes, either strictly sharper than existing bounds under standard moment, source, and capacity assumptions or valid under weaker assumptions.}We also establish a lower bound that matches our upper bounds in certain regimes up to a polylogarithmic factor. In the noisy case, we provide an instance-based comparison between SGDIR and ridge regression. Under general assumptions and a mild lower bound on the regularization parameter, we show that the expected excess risk of SGDIR is no larger than that of ridge regression, up to a polylogarithmic factor. Numerical experiments on synthetic and real data are consistent with  our theoretical findings.
\end{abstract}
Keywords: stochastic gradient descent, noiseless model, ridge regression, least-squares
\section{Introduction}
A central question in machine learning is how to fit a model from copies of a feature-response pair \((x,y)\)  so that it generalizes well to unseen data. Beyond generalization performance, the running time, space complexity, and degree of parallelism of the learning algorithm are also important considerations in practice.
  In this paper, we assume that  \((x,y)\) is a square-integrable random vector in  \(\mathcal{H}\times\mathbb{R}\), where  \(\mathcal{H}\) is a separable Hilbert space over \(\mathbb{R}\), equipped with the inner product \(\langle \cdot, \cdot \rangle\).   Define the  loss function \(L(\theta):=\frac{1}{2}E[(y-\langle x,\theta\rangle)^{2}]\)  for  any \(\theta\in\mathcal{H}\), and let  \((x_{t},y_{t})_{t\geq0}\) be  independent copies of \((x,y)\). Our objective is to approximately minimize \(L\) using the training sequence  \((x_{i},y_{i})\), \(0\leq i\leq n-1\), where \(n\) is the sample size. This problem has been extensively studied using two main approaches: stochastic gradient descent (SGD)  and ridge (or least-squares) regression.

The asymptotic properties of averaged SGD with constant step size were first studied by \citet{ruppert1988efficient} and \citet{polyak1992acceleration}. In the  non-strongly convex setting, \citet{bachMoulines2013non} show that the expected excess risk of averaged SGD is \(O(1/n)\). Under strong convexity of \(L\), \citet{jain2018parallelizing} analyze tail-averaged and mini-batch SGD in Euclidean spaces and derive asymptotically optimal bounds. More recently, \citet{kahale2026UnbiasedLeastSquares} introduces an unbiased SGD-like estimator of the minimizer of \(L\) and derives convergence bounds that match those of \citet{bachMoulines2013non} up to a polylogarithmic factor.   
 
The aforementioned studies concern finite-dimensional settings. In practice,  the number of features can be very large, or even infinite, as in reproducing kernel Hilbert spaces (RKHSs). We next review work  establishing dimension-independent bounds on the expected excess risk.  In the  RKHS framework, \citet{vito2007optimal} and \citet{BachDieuleveut2016} derive convergence bounds for kernel ridge regression (KRR) and averaged SGD, respectively, that are asymptotically optimal under suitable assumptions.     \citet{BachLeastSquaresJMLR2017} derive similar bounds for regularized averaged SGD in Hilbert spaces. Improved bounds on the expected excess risk in hard learning problems have been established by  \citet{BachPillaudNIPS2018}   using SGD with multiple passes and by \citet{jun2019kernel}   using a variant of KRR.  Likewise, \citet{mucke2019beating} derive improved bounds for tail-averaged SGD in Hilbert spaces under certain regularity assumptions and also study minibatching.
Finally, under suitable moment assumptions, 
\citet{kakade2023benign}  derive
 sharp excess risk bounds for  constant-step-size SGD with iterate averaging or tail averaging  in terms of the full eigenspectrum of the covariance operator.

Motivated by the double-descent phenomenon and the near-zero training loss achieved by over-parameterized neural networks, a growing literature has analyzed SGD, linear regression, and ridge regression in over-parameterized settings. In particular, the last iterate of constant-step-size SGD has received considerable attention in this context, given the success of SGD in training over-parameterized models.
For instance, \citet{bartlett2020benign} and \citet{hastie2022surprises} analyze the prediction accuracy of the minimum-norm interpolator in over-parameterized linear regression, while \citet{tsigler2023benign} provide analogous results for ridge regression. \citet{zdeborova2022generalization} analyze the performance of KRR under a low-noise Gaussian design. \citet{ma2018power} establish exponential convergence bounds for the last iterate of SGD  under noiselessness and strong convexity. In a noiseless Hilbert space setting, \citet{bach2020tight} and \citet{Flammarion2021last} derive tight convergence bounds for non-averaged SGD.  \citet{Kakade2024scaling} provide  sharp convergence bounds for the  test error of SGD in terms of the number of parameters and training samples in an infinite dimensional linear regression
setup. In zero- or low-noise settings, \citet{attia2026fast} establish bounds on the expected excess risk of the last iterate of constant-step-size SGD for smooth convex objectives.

\subsection{Contributions}
Dimension-free convergence rates for ridge regression and SGD are often established under \emph{capacity} and \emph{source} conditions, which are formally defined in Section~\ref{sub:Assumptions}. In this paper, we study SGDIR, an      SGD method with initial regularization.
Specifically, given a regularization parameter \(\Lambda > 0\), a constant step size \(\gamma > 0\), and an integer \(m \geq 3\), we generate the sequence \((\theta_{t})_{t \geq 0}\) by the recursion
\begin{equation}\label{eq:recTheta}
\theta_{t+1} = (1 - \gamma \lambda_{t}) \theta_{t} - \gamma \big( \langle x_{t}, \theta_{t} \rangle - y_{t} \big) x_{t},
\end{equation}
where \(\lambda_{t} = \mathbf{1}(t < m)\Lambda\) and the initial vector is \(\theta_{0} = 0\). 
We  use the tail-averaged estimator  \(\bar\theta_{2m,3m}\) to approximately minimize \(L\), where  \(\bar u_{i:j}:=(j-i)^{-1}\sum^{j-1}_{t=i}u_{t}\) for  \(0\leq i< j\) and any sequence \((u_{t})_{t\geq0}\)  in \(\mathcal{H}\).   In the particular case where  \(\sigma=0\), we will also use \(\theta_{N}\) to approximately minimize \(L\), where \(N\) is drawn uniformly at random in \(\{2m,\dots,3m-1\}\).   

Our main contributions are as  follows. Let  \(\lambda\) denote the regularization parameter in   ridge regression. \begin{enumerate}
\item In the noiseless case, under appropriate capacity and source conditions, we establish dimension-free upper bounds on the expected excess risk of   \(\bar\theta_{2m,3m}\) and  \(\theta_{Q}\)   that, up to a \(\log^{2}m\)  factor, are no worse than existing bounds and are strictly sharper in certain regimes. In particular, for a certain value of the source parameter, we obtain bounds of order \(m^{-2}\log^{2}m\). Moreover, for any \(\epsilon>0\),  we obtain bounds of order  \(m^{-3+\epsilon}\) for  another value of the source parameter, provided that the capacity parameter is of order \(\epsilon^{-1}\) or larger. To the best of our knowledge, no previous algorithm has been shown to achieve such a tradeoff between sample size and accuracy under comparable assumptions.\item 
Under general assumptions, we show that, for every problem instance with  \(n\)  training samples and every regularization parameter  \(\lambda\) of order \(1/n\) or larger, explicit choices  of \(\gamma\) and \(\Lambda\) ensure that the expected excess risk of SGDIR is at most a factor of \(\log^{2}n\) larger than that of ridge regression. This factor improves to    \(\log n\) when \(\lambda\) is of order \(\sqrt{\log n}/n\) or larger, and to a constant    in the noiseless case when  \(\lambda\) is of order   \((\log n)/n\) or larger. Such choices of  \(\lambda\) are common in the analysis of ridge regression (e.g.,   \citet{vito2007optimal} and  \citet[Section 7.6.6]{bach2024learning}). Our proof technique also yields a lower bound on the expected excess risk of ridge regression that is of independent interest. \item In the noiseless case, we derive a minimax lower bound on the expected excess risk of any estimator of a minimizer of \(L\) based on \(n\) training samples. Our lower bound is new, to the best of our knowledge. For a range of values of the source parameter, it matches our upper bound up to a polylogarithmic factor, and its exponent depends only on the source parameter.
\end{enumerate}

\subsection{Other related work}

\comment{When \(\mathcal{H}\) has finite dimension \(d\), standard implementations of  SGD and SGDIR based on \(n\) samples require \(O(nd)\) time and \(O(d)\) space, whereas a standard implementation of ridge regression takes \(O(nd^{2}+d^{3})\) time and \(O(d^{2})\) space.} In an RKHS, assuming that each kernel evaluation takes constant time,  SGD based on \(n\) samples can be implemented  using \(O(n^{2})\) time and \(O(n)\) space \cite[Section 7.4.5]{bach2024learning}, and a straightforward adaptation of the same  approach yields the same complexities for SGDIR. In contrast, the standard algorithm for kernel ridge regression   requires \(O(n^{3})\) time and \(O(n^{2})\) space \cite[Eq.~7.7]{bach2024learning}. More efficient random-feature implementations of KRR and SGD are analyzed for a class of kernels by \citet{rudi2017generalization} and \citet{rosasco2018learning}, respectively, and alternative implementations of KRR can be found  in \citet[Section 7.4]{bach2024learning}.

Recent work has compared the performance of ridge regression, SGD, and gradient descent. For example, \citet{kakade2021benefits} provide an instance-based comparison of the generalization error of tail-averaged SGD and ridge regression in several settings, including one-hot and Gaussian distributions.
In these settings, they show that tail-averaged SGD generalizes no worse than ridge regression, up to a logarithmic factor. In contrast, our comparison between SGD and ridge regression does not make distributional assumptions on  \(x\) beyond standard moment conditions, but requires  a lower bound on the regularization parameter.      \citet{Kakade2025risk} establish conditions under which gradient descent dominates ridge regression and is incomparable with SGD.
 Beyond machine learning, kernel methods have a range of applications, including optimization    \cite{bertsimas2022data} and numerical integration \cite{rosacso2025efficient}.
Our lower bounds for ridge regression are inspired from the financial engineering literature on model-free bounds, which has been the subject of numerous studies (e.g., \cite{BP02,kahale2017}). 

The rest of the paper is organized as follows. Section~\ref{se:Notation} introduces our notation and main assumptions. Section~\ref{se:ExpectedRisk} derives upper bounds on the expected excess risk of SGDIR. Section~\ref{se:SourceCapacity} establishes convergence bounds under source and capacity assumptions and provides a more technical comparison with the prior literature. Section~\ref{se:ridgeComparison} compares SGDIR with ridge regression. Section~\ref{se:lowerBound} derives a lower bound on the expected excess risk of any algorithm. Section~\ref{se:numer} presents numerical experiments, and Section~\ref{se:conclusion} concludes. Omitted proofs are provided in the appendix.
\section{Notation and main Assumptions}\label{se:Notation}
For \(u, v \in \mathcal{H}\),  denote by \(u \otimes v\) the linear operator on \(\mathcal{H}\) defined by
 \((u\otimes v)w=\langle v,w\rangle u\), for \(w\in\mathcal{H}\). Thus, when \(\mathcal{H}\) is equal to \(\mathbb{R}^{d}\) equipped with the usual scalar product \(\langle u, v \rangle=u^{T}v\), the matrix of  \(u \otimes v\)  is \(uv^{T}\).
If \(A\) and \(B\) are self-adjoint operators on \(\mathcal{H}\), we say that \(A\preccurlyeq B\) (resp.  \(A\succcurlyeq B\)) if  the operator \(B-A\) (resp. \(A-B\)) is positive semidefinite. For any operator $M$ on \(\mathcal{H}\),  we use the shorthand $M+\lambda$ to denote $M+\lambda I$.
Given a self-adjoint positive semidefinite operator \(M\) on \(\mathcal{H}\), let \(||v||_{M}:=\sqrt{\langle v, Mv\rangle}\) for \(v\in\mathcal{H}\). If \(\psi\) is a square-integrable random vector in  \(\mathcal{H}\), set \(\var_{M}(\psi)=E[||\psi||^{2}_{M}]-||E[\psi]||^{2}_{M}\).   Standard results show that  the noncentered covariance operator  \(\Sigma:=E[x\otimes x]\)  is   positive semidefinite, self-adjoint, and  trace-class,  with \(\tr(\Sigma )=E[||x||^{2}]\). Consequently, there exists an orthonormal basis \((e_i)_{i\in J}\) of \(\mathcal H\) consisting of eigenvectors of \(\Sigma\), where \(J=\{1,\dots,d\}\) if \(\mathcal{H}\) is \(d\)-dimensional and \(J=\mathbb{N}-\{0\}\) otherwise. Thus  \(\Sigma e_{i}=\mu_{i}e_{i}\)  for \(i\in J\), where \(\mu_{i}\geq0\), and \(\sum_{i\in J}\mu_{i}=\tr(\Sigma)\). We  suppose that the \(\mu_{i}\)'s are sorted in nonincreasing order.\comment{When \(J\) is infinite, \(\mu_{i}\) goes to \(0\) as \(i\) goes to infinity.}

For the remainder of this section and throughout Sections~\ref{se:ExpectedRisk} and \ref{se:SourceCapacity}, we assume the following.

    \begin{description}
\item[Assumption A1.] The function \(L\) attains its minimum at a  vector \(\theta^{*} \in \mathcal{H}\), and there are constants \(R > 0\) and \(\sigma\geq0\) such that\begin{equation} \label{eq:R2Assumption}
E\left[\|x\|^2 x\otimes x \right] \preccurlyeq\ R^2 \Sigma,
\end{equation}
\begin{equation}\label{eq:SigmaAssumption}
E[(y - \langle x, \theta^{*} \rangle)^{2} x\otimes x]\preccurlyeq\sigma^{2}\Sigma.
\end{equation}   
\end{description}
Because \(L\) attains its minimum at \(\theta^{*}\), a standard calculation shows the  \emph{normal equation} 
\begin{equation}\label{eq:NormalEq}
\Sigma\theta^{*}=E[yx].
\end{equation}Conditions    \eqref{eq:R2Assumption}   and   \eqref{eq:SigmaAssumption}  are taken from \citet{bachMoulines2013non}. It is easy to verify   \eqref{eq:R2Assumption}  if \(\|x\|\leq R\) with probability \(1\). It follows from   \eqref{eq:R2Assumption}  that \(\tr\left[\Sigma\right]\leq R^{2}\) \cite{BachLeastSquaresJMLR2017}. Hence\begin{equation}\label{eq:SigmaLeR^2}
0\preccurlyeq\Sigma\preccurlyeq R^{2}I.
\end{equation}Condition  \eqref{eq:SigmaAssumption} holds if \(|y - \langle x, \theta^{*} \rangle|\leq\sigma\) with probability \(1\). Moreover, if the error term \(y - \langle x, \theta^{*} \rangle\) is independent of \(x\), then  \eqref{eq:SigmaAssumption}   holds with \(\sigma^{2}=E[(y - \langle x, \theta^{*} \rangle)^{2} ]\).  

We further assume that  \(\gamma\) and \(\Lambda\) satisfy the bound\begin{equation}\label{eq:GammaUpBound}
\gamma(R^{2}+\Lambda)\leq1.
\end{equation}
Condition~\eqref{eq:GammaUpBound} is approximately a factor of 2 weaker than a corresponding condition in \citet[Theorem~2]{BachLeastSquaresJMLR2017} for regularized SGD.
 Since \(\gamma\Lambda\rightarrow0\) as \(m\rightarrow\infty\) in our main results,    \eqref{eq:GammaUpBound} essentially requires \(\gamma\) to be of order \(1/R^{2}\) or smaller.  
  Similar step-size conditions are commonly imposed in the analysis of SGD algorithms for least-squares regression \cite{bachMoulines2013non,BachDieuleveut2016,kahale2026UnbiasedLeastSquares}.

For \(\theta\in\mathcal{H}\), let \(\mathcal{R}(\theta):=L(\theta)-L(\theta^{*}) \) denote the excess risk of \(\theta\). Using   \eqref{eq:NormalEq},  a standard calculation shows that   \(\mathcal{R}(\theta)=\frac{1}{2}||\theta-\theta^{*}||_{\Sigma}^{2}\)   and that   the regularized loss function  \(L_{\lambda}(\theta):=L(\theta)\) + \(\lambda||\theta||^{2}/2\) minimized  at \(\theta_{\lambda}:= (\Sigma+\lambda )^{-1}\Sigma\theta^{*}\).
When \(t < m\) (resp. \(t \geq m\)), the recursion~\eqref{eq:recTheta} corresponds to the standard SGD  applied to \(L_{\Lambda}\) (resp. \(L\)). This observation provides  intuition behind our approach.  Assume for simplicity that \(\sigma=0\).  The smallest eigenvalue of the Hessian matrix  \(\Sigma+\Lambda\) of  \(L_{\Lambda}\)  is  at least \(\Lambda\). By  \cite[Theorem 1]{jain2018parallelizing}) and  \cite[Proposition SM.1]{kahale2026UnbiasedLeastSquares},  this suggests that \(E[||\theta_{m}-\theta_{\Lambda}||^{2}]\) is upper bounded by an exponentially decreasing function of \(m\), with a rate proportional to \(\Lambda\). Thus, if \(\Lambda\) is sufficiently large,  the last \(2m\) steps of our algorithm correspond to an unregularized SGD with an initial state ``close'' to  \(\theta_{\Lambda}\). Moreover, if \(\Lambda\) is sufficiently small,   \(\theta^{*}\)  is intuitively  ``closer'' to  \(\theta_{\Lambda}\) than to \(0\). This suggests that, for a suitable choice of \(\Lambda\), our algorithm forgets initial conditions faster  than an unregularized SGD with initial state  \(0\) and \(2m\) steps. We stress that the above argument is informal and will not be used in our formal derivations.

Let \(\psi\) be a (possibly biased)
square-integrable estimator of \(\theta^{*}\), and, for \(m'>0\), let  \(\psi^{(m')}\) be the average of \(m'\) independent copies of \(\psi\). We can bound the expected excess risk of \(\psi\)  by bounding separately each term of the standard bias-variance  decomposition \(2E[\mathcal{R}(\psi)]=||E[\psi]-\theta^{*}||_{\Sigma}^{2}+\var_{\Sigma}(\psi)\). We refer to the first term as the squared bias of \(\psi\) and to the second as its variance. Since \begin{equation}\label{eq:modelAvg}
2E[\mathcal{R}(\psi^{(m')})]=||E[\psi]-\theta^{*}||_{\Sigma}^{2}+\frac{\var_{\Sigma}(\psi)}{m'}, \end{equation}this approach  also provides a bound on the  expected excess risk of \(\psi^{(m')}\).  
\subsection{Reproducing kernel Hilbert space}\label{sub:kernel}
We assume in this subsection that \(\mathcal{H}\) is a separable  RKHS on \(\mathbb{R}\). Thus, the elements of  \(\mathcal{H}\) are functions from a set \(\mathcal{X}\) to \(\mathbb{R}\),  and there is a positive definite function \(K: \mathcal{X}\times \mathcal{X}\rightarrow\mathbb{R}\) such that \(K(a,.)\in\mathcal{H}\) and \(f(a)=\langle f, K(a,.)\rangle\) for  \(f\in\mathcal{H}\) and \(a\in \mathcal{X}\). 
Assume that \(K(a,a)\leq R^{2}\) for \(a\in\mathcal{X}\), where \(R>0\) is a constant, and that \(\mathcal{X}\) is endowed with a \(\sigma\)-algebra such that \(z\mapsto K(z,z)\) is a measurable function on \(\mathcal{X}\).
Let \((X,y)\) be a random pair taking values in \(\mathcal{X\times \mathbb{R}}\) such that \(y\) is square-integrable. Assume there is   \(f_{*}\in\mathcal{H}\) with \(E[y|X]=f_{*}(X)\) and\begin{equation}
E[(y-f_{*}(X))^{2}|X]\leq \sigma^{2},
\end{equation}  where  \(\sigma\geq 0\) is a constant. Finally, let  \(x=K(X,.)\) be the random function induced by \(X\).  A standard calculation shows that \((x,y)\) satisfies Assumption A1, with \(\theta^{*}=f_{*}\), and that   \(L(f)=\frac{1}{2}E[(y-f(X))^{2}]\) for \(f\in \mathcal{H}\).
Thus  approximately minimizing \(L\) amounts to approximating  \(y\) by the random variable   \(f(X)\), where \(f\in\mathcal{H}\). 
The assumption that \(\theta^{*}\in\mathcal{H}\) was made to avoid additional technicalities. Modeling frameworks that do not require  \(\theta^{*}\in\mathcal{H}\) are common in the  RKHS  literature (e.g., \citet{BachDieuleveut2016}).  The terms involving \(\theta^{*}\)  in the bounds of Lemma \ref{le:TotalBiasAtT}, Theorem~\ref{th:noiseless}, Proposition \ref{pr:SimpleBound}, Theorem~\ref{th:Main} and Theorem~\ref{th:Ridge} are upper bounded by \(\mathcal{R}(0)\), up to an absolute constant factor. These results, as well as Theorem~\ref{th:noiselessArAc}, can be extended to the case where  \(\theta^{*}\notin\mathcal{H}\) by adapting the approach of \citet{BachDieuleveut2016}.    
\section{Bounding the expected excess risk}\label{se:ExpectedRisk}
In this section, we first define a noiseless version  \(w_{t}\) of \(\theta_{t}\),  \(t\geq0\),    and establish three alternative bounds on the  expected excess risk of \(w_{N}\). We then provide a bound on the expected excess risk of \(\bar\theta_{2m:3m}\). \subsection{The noiseless process}Define the sequence \(w_{t}\), \(t\geq0\), by the recursion \begin{equation}\label{eq:RecW_t}
w_{t+1}=(1-\gamma\lambda_{t})w _{t}-\gamma \langle x_{t},w _{t}-\theta^{*}\rangle x_{t},
\end{equation} 
with initial condition \(w_{0}=0\). This recursion is obtained by replacing \(y_{t}\) with \(\langle x_{t},\theta^{*}\rangle\) in \eqref{eq:recTheta}.  Hence \(w_{t} = \theta_{t}\) almost surely for  \(t \ge 0\) when \(\sigma = 0\). Lemma~\ref{le:TotalBiasAtT} provides a simple bound on the expected excess risk of \(w_N\).

\begin{lemma}\label{le:TotalBiasAtT}
Suppose \(m \ge 3\), Assumption~A1 and \eqref{eq:GammaUpBound} hold, and \(\Lambda \ge \log(m)/(\gamma m)\). Then
\begin{equation}\label{eq:TotalBiasAtT}
 E[\mathcal{R}(w_N)]\leq8\Lambda^{2}||(\Sigma+\Lambda)^{-1}\theta^{*}||^{2}_{\Sigma}.
\end{equation}
\end{lemma} 
Replacing \(\theta_{\Lambda}\) by its value shows that \(\Lambda^{2}||(\Sigma+\Lambda)^{-1}\theta^{*}||_{\Sigma}^{2}=2\mathcal{R}(\theta_{\Lambda})\).
Thus, \eqref{eq:TotalBiasAtT} implies that \(E[\mathcal{R}(w_N)]\leq16\mathcal{R}(\theta_{\Lambda})\). 
This is consistent with our informal  argument in Section \ref{se:Notation} that suggests that \(w_m\) is ``close'' to \(\theta_{\Lambda}\). Indeed, we expect in this case that the excess risk at subsequent time steps not to be much larger than that of \(\theta_{\Lambda}\). 

Set\begin{equation*}
\tilde\Lambda= \frac{1}{\gamma m}\tr\left[\Sigma^{2}\left(\Sigma+\frac{1}{4\gamma m}\right)^{-2}\right].
\end{equation*}
Because \(\Sigma^{2}\left(\Sigma+(4\gamma m)^{-1}\right)^{-2}\leq I\), we have \(\tilde\Lambda\leq d/(\gamma m)\) if \(\mathcal{H}\) has finite dimension \(d\). More generally, as shown in Section~\ref{se:SourceCapacity},  \(\tilde\Lambda\)  tends to be ``small'' when \(\mu_{i}\) decays ``rapidly'' as \(i\) increases.  Quantities related to   \(\tilde\Lambda\)   such as the effective dimension often appear in the analysis of SGD (e.g., \citet{kakade2023benign}).
Lemma~\ref{le:TotalBiasAtTBis} provides an alternative bound on  \(E[\mathcal{R}(w_N)]\) that depends on \(\tilde{\Lambda}\).        
\begin{lemma}\label{le:TotalBiasAtTBis} Under the assumptions of Lemma~\ref{le:TotalBiasAtT}, we have
\begin{equation}\label{eq:TotalBiasAtTBis}
 E[\mathcal{R}(w_N)]\leq 4||(I-\gamma \Sigma )^{m}\theta^{*}||_{\Sigma}^{2}+8\gamma  R^{2}\Lambda^{2}\tilde\Lambda||(\Sigma+\Lambda)^{-1}\theta^{*}||^{2}.
\end{equation}
\end{lemma}
Up to a constant factor, the first term in the righthand side of \eqref{eq:TotalBiasAtTBis} is equal to the excess risk of gradient descent at time step \(m\), with initial state \(0\) \cite[Section 5.2.1]{bach2024learning}. The bounds  \eqref{eq:TotalBiasAtT} and  \eqref{eq:TotalBiasAtTBis} are in general not directly comparable. 
Theorem~\ref{th:noiseless} provides a third bound on  \(E[\mathcal{R}(w_N)]\) that essentially combines  \eqref{eq:TotalBiasAtT} and \eqref{eq:TotalBiasAtTBis}.   
   
\begin{theorem}\label{th:noiseless}
Under the assumptions of Lemma~\ref{le:TotalBiasAtT}, we have
\begin{equation}\label{eq:TotalBiasAtTCombined}
  E[\mathcal{R}(w_N)]\leq4||(I-\gamma \Sigma )^{m}\theta^{*}||_{\Sigma}^{2}+16\gamma R^{2}\Lambda^{2}\tilde\Lambda ||(\Sigma+\Lambda)^{-1}(\Sigma+\tilde\Lambda)^{-1/2}\theta^*||^{2}_{\Sigma}.
\end{equation}
\end{theorem}
 \comment{In particular, under the conditions of Theorem \ref{th:noiseless}, the squared distance between  \(\theta_{\Lambda}\) and  the \(t\)-th iterate of  standard gradient descent with regularization parameter \(\Lambda\) and step size \(\gamma\) is of order \(m^{-2}\) or smaller  for \(t\geq m\).}  

Using \eqref{eq:GammaUpBound} and the inequality \(mz(1-z)^m \leq 1\) for \(0 \leq z \leq 1\), a standard diagonalization argument shows that, up to a constant factor, the righthand side of \eqref{eq:TotalBiasAtTCombined} is no larger than those of \eqref{eq:TotalBiasAtT} and \eqref{eq:TotalBiasAtTBis}. On the other hand, by convexity of the excess risk,\begin{equation}\label{eq:ConvExcessRisk}
E[\mathcal{R}(\bar w_{2m:3m})]\le E[\mathcal{R}(w_N)].
\end{equation} 
Thus, under the assumptions of Lemma~\ref{le:TotalBiasAtT}, \(E[\mathcal{R}(\bar w_{2m:3m})]\) is bounded above by each of the righthand sides of \eqref{eq:TotalBiasAtT}, \eqref{eq:TotalBiasAtTBis}, and \eqref{eq:TotalBiasAtTCombined}.
\subsection{Incorporating noise}
To account for noise, set   \(\delta_{t}:=\theta_{t}-w_{t}\) and \(v_{t}:=\gamma(y_{t}-\langle x_{t},\theta^{*}\rangle)x_{t}\) for  \(t\geq0\).  By \eqref{eq:NormalEq}, we have \(E[v_{t}]=0\) and, by \eqref{eq:recTheta} and \eqref{eq:RecW_t}, \begin{equation}\label{eq:RecDelta}
\delta_{t+1}=(I-\gamma(  x_{t}\otimes x_{t}+\lambda_{t} I))\delta_{t}+v_{t}.
\end{equation}
 As \(x_{t}\) and \(\delta_{t}\) are independent, it follows by induction that \(E[\delta_{t}]=0\) for \(t\geq0\). 
Lemma \ref{le:BiasVarDecomp} provides a bound on \(\var_{\Sigma}(\bar\delta_{2m:3m})\). 
      \begin{lemma}\label{le:BiasVarDecomp} Suppose \(m \ge 3\), Assumption~A1 and \eqref{eq:GammaUpBound} hold. Then
\(\var_{\Sigma}(\bar\delta_{2m:3m})\leq  8\gamma\sigma^{2}\tilde\Lambda\). \end{lemma}
When the dimension of \(\mathcal{H}\) is finite and equal to \(d\), the bound is of order \(\sigma^{2}d/m\), that frequently appears in the analysis of SGD for least-squares problems \cite{bachMoulines2013non,jain2018parallelizing}. 

The preceding results lead to the following simple bound on \(E\left[\mathcal{R}(\bar\theta_{2m:3m})\right]\).
  \begin{proposition}\label{pr:SimpleBound}Under the assumptions of Lemma~\ref{le:TotalBiasAtT}, we have
\begin{displaymath}
E\left[\mathcal{R}(\bar\theta_{2m:3m})\right]\leq16\Lambda^{2}||(\Sigma+\Lambda)^{-1}\theta^{*}||^{2}_{\Sigma}+8\gamma\sigma^{2}\tilde\Lambda.
\end{displaymath}
\end{proposition}
\begin{proof}
As \(||u+v||^{2}_{\Sigma}\leq 2||u||^{2}_{\Sigma}+2||v||^{2}_{\Sigma}\) for \(u,v\in\mathcal{H}\), we have \begin{eqnarray*}E[||\bar\theta_{2m:3m}-\theta^{*}||_{\Sigma}^{2}]&\leq&2E[||\bar w_{2m:3m}-\theta^{*}||_{\Sigma}^{2}]+2E[||\bar \delta_{2m:3m}||_{\Sigma}^{2}]\\&\le&32\Lambda^{2}||(\Sigma+\Lambda)^{-1}\theta^{*}||^{2}_{\Sigma}+16\gamma\sigma^{2}\tilde\Lambda,\end{eqnarray*}where the second inequality follows from  \eqref{eq:TotalBiasAtT}, \eqref{eq:ConvExcessRisk}, and Lemma \ref{le:BiasVarDecomp}.   
\end{proof}
Building on the proofs of Theorem~\ref{th:noiseless} and Lemma~\ref{le:BiasVarDecomp}, Theorem~\ref{th:Main} yields bounds on the squared bias and variance of \(\bar{\theta}_{2m:3m}\), and consequently on the expected excess risk of \(\bar{\theta}_{2m:3m}\) and of averages of its independent copies.

\begin{theorem}\label{th:Main} Under the assumptions of Lemma~\ref{le:TotalBiasAtT}, we have

\begin{equation}\label{eq:BiasOfAvgBound}
||E[\bar\theta_{2m:3m}]-\theta^{*}||^{2}_{\Sigma}\le 4||(I-\gamma \Sigma )^{m}\theta^{*}||_{\Sigma}^{2},
\end{equation}and \begin{equation}
\label{eq:VarOfAvgBound}\var _{\Sigma}(\bar\theta_{2m:3m})\leq    128\gamma R^{2}\Lambda^{2}\tilde\Lambda ||(\Sigma+\Lambda)^{-1}(\Sigma+\tilde\Lambda)^{-1/2}\theta^*||_{\Sigma}^{2}+16\gamma\sigma^{2}\tilde\Lambda.
\end{equation}Moreover, for \(m'\geq1\),\begin{equation}\label{eq:totalError}
E\left[\mathcal{R}(\bar\theta_{2m:3m}^{(m')})\right]=\frac{1}{2} ||E[\bar\theta_{2m:3m}]-\theta^{*}||^{2}_{\Sigma}+\frac{1}{2m'}\var _{\Sigma}(\bar\theta_{2m:3m}).
\end{equation}
\end{theorem}
The relation \eqref{eq:totalError} is a special case of \eqref{eq:modelAvg} and yields a bound on \(E[\mathcal{R}(\bar\theta_{2m:3m})]\) when \(m'=1\). Here again, a diagonalization argument shows that this bound is no worse than that of Proposition \ref{pr:SimpleBound}, up to a constant factor.  When \(\sigma=0\), this bound coincides with the righthand side of  \eqref{eq:TotalBiasAtTCombined}, up to a constant factor.

\begin{remark} For any constant \(a>2\), replacing \(\bar{\theta}_{2m:3m}\) by \(\bar{\theta}_{2m:am}\) preserves  \eqref{eq:BiasOfAvgBound}, while  \eqref{eq:VarOfAvgBound} continues to hold up to a constant factor that depends on \(a\). Likewise,  \eqref{eq:TotalBiasAtT}, \eqref{eq:TotalBiasAtTBis}, and  \eqref{eq:TotalBiasAtTCombined} still hold, up to a constant factor that depends on \(a\), if \(N\) is uniformly chosen in \(\{2m,\dots,am-1\}\). \end{remark}

\section{Convergence results}\label{se:SourceCapacity}
 To establish convergence rates for the expected excess risk of SGDIR and compare them with the existing literature, we state capacity and source assumptions in Section~\ref{sub:Assumptions}. These assumptions, and related variants, are commonly used in the analysis of SGD  and ridge or kernel regression for least-squares problems \cite{vito2007optimal,BachDieuleveut2016,rudi2017generalization,BachLeastSquaresJMLR2017,mucke2019beating,bach2020tight,Flammarion2021last}. In Section~\ref{sub:implications}, we use  Theorems~\ref{th:noiseless} and \ref{th:Main} to derive convergence results under these assumptions. In Section \ref{sub:comparison}, we compare our results with previously established bounds.
Throughout this section, we focus on the noiseless case. Convergence results can also be derived for    \(\bar\theta_{2m,3m}\)  in the general setting but, up to a polylogarithmic factor,  they match those of  \citet[Theorem 2]{BachDieuleveut2016} for \(r\leq0\) and  those of  
  \citet[Corollary 1]{mucke2019beating} for \(r\geq0\).    \subsection{Capacity and Source assumptions}\label{sub:Assumptions}For \(r\geq0\), we can define \(\Sigma^{r}\) as the unique linear operator such that \(\Sigma^{r}e_{i}=\mu_{i}^{r}e_{i}\) for \(i\in J\). By convention, for \(r<0\) and \(v\in\mathcal{H}\), the equality  \(\theta^{*}=\Sigma^{r}v\) means that  \(v=\Sigma^{-r}\theta^{*}\).\begin{description}
\item[Assumption Ac.] We have \(\mu_{i}\le ci^{-\alpha}\) for \(i\in J\), where \(\alpha\geq1\) and \(c\geq0\) are constants.
\end{description}
The parameter \(\alpha\) quantifies the strength of Assumption~Ac. The larger \(\alpha\) is, the stronger the assumption.  As \(\tr(\Sigma)\geq i\mu_{i}\) for any \(i\in J\), Assumption Ac always holds with \(\alpha=1\) and \(c=\tr(\Sigma)\). When \(\mathcal{H}\) has a finite dimension \(d\), Assumption Ac holds for any \(\alpha>1\) and  \(c=\mu _{0}d^{\alpha}\). Intuitively, large values of \(\alpha\) correspond to data distributions with low effective dimension, leading to faster convergence rates for SGD-like algorithms.   The slightly stronger capacity condition \(\tr(\Sigma^{1/\alpha})<\infty\) is often used in the literature \cite{BachLeastSquaresJMLR2017,jun2019kernel}.  \begin{description}
\item[Assumption As.] There is \(v^{*}\in \mathcal{H}\) and a real number \(\rho^{*}\) such that \(\theta^{*}=\Sigma^{r}v^{*}\), with \(r\geq-1/2\) and \(||v^{*}||^{2}\leq \rho^{*}\).
\end{description}

Assumption~As always holds for \(r=-1/2\) and \(\rho^{*}=2\mathcal{R}(0)\). The parameter \(r\) measures the strength of Assumption~As: the larger \(r\) is, the stronger the assumption. Indeed, suppose that  \(\theta^{*}=\Sigma^{r}v^{*}\), with \(r\geq-1/2\) and \(||v^{*}||^{2}\leq \rho^{*}\), and let \(r'\in[-1/2,r]\). Then \(\theta^{*}=\Sigma^{r'}(\Sigma ^{r-r'}v^{*})\), and \(||\Sigma ^{r-r'}v^{*}||\leq R^{2r-2r'}||v^{*}||\) by \eqref{eq:SigmaLeR^2}. Thus, Assumption As holds  for \((r',R^{4r-4r'}\rho^{*})\).     
Intuitively, a large value of \(r\) indicates strong alignment of \(\theta^{*}\) with the important directions of \(\Sigma\), making the distribution of \((x,y)\) easier to fit. Assumptions Ac and As are also linked to the growth rate of the coordinates of \(\theta^{*}\) in the basis  \((e_{i})_{i\in J}\). Indeed, assume that \(J=\mathbb{N}-\{0\}\), and there are constants \(\alpha\geq1\), \(c>c'>0\), \(r\geq-1/2\),  \(\delta>1+2\alpha r\),  and \(W>0\)  such that \(c'i^{-\alpha}\le\mu_{i}\le ci^{-\alpha}\) and \(\langle \theta^{*},e_{i}\rangle i ^{\delta/2}\leq W\) for \(i\geq1\). Then Assumption As holds for \(r\) and a suitable choice of \(\rho^{*}\) \cite[Section 2.7]{BachDieuleveut2016}. Further discussion of the capacity and source assumptions, as well as their connections to Sobolev spaces, can be found in \cite{bach2024learning,bach2020tight,rosacso2025efficient}. Lemmas~\ref{le:BoundCapacity} and \ref{le:RegBounds} provide bounds on the quantities appearing in Theorems~\ref{th:noiseless} and \ref{th:Main} under Assumptions~Ac and As, respectively.    
\begin{lemma}\label{le:BoundCapacity}
Under Assumption Ac, we have \(\gamma m \tilde\Lambda\leq(4c\gamma m)^{1/\alpha}\).  
 \end{lemma}
\begin{proof}
 
Let  \(\lambda>0\). As \(\lambda\Sigma^{2}\left(\Sigma+\lambda\right)^{-2}\preccurlyeq\Sigma\), we have \(\lambda\tr\left[\Sigma^{2}\left(\Sigma+\lambda\right)^{-2}\right]\leq\tr[\Sigma]\leq R^{2}\). Moreover, under Assumption Ac, \begin{displaymath}
\tr\left[\Sigma^{2}(\Sigma+\lambda)^{-2}\right]\leq\sum^{\infty}_{i=1}\left(\frac{ci^{-\alpha}}{ci^{-\alpha}+\lambda}\right)^{2}\leq\int^{\infty}_{0}\left(\frac{c}{c+\lambda v^{\alpha}}\right)^{2}dv=\frac{1}{\alpha}\left(\frac{c}{\lambda}\right)^{1/\alpha}\int^{\infty}_{0}\frac{u^{1/\alpha-1}}{(1+u)^{2}}du.\end{displaymath} 
We have\begin{displaymath}
\int^{\infty}_{0}\frac{u^{1/\alpha-1}}{(1+u)^{2}}du=2\alpha\int^{\infty}_{0}\frac{u^{1/\alpha}}{(1+u)^{3}}du=2\alpha\int^{1}_{0}\frac{u^{1/\alpha}+u^{1-1/\alpha}}{(1+u)^{3}}du\leq\alpha,
\end{displaymath}where the first equation follows by integration by parts, the second by a change of variables, and the third from the inequality \(u^{1/\alpha}+u^{1-1/\alpha}\leq1+u\) for \(0<u\leq1\).
This yields \(\tr\left[\Sigma^{2}(\Sigma+\lambda)^{-2}\right]\leq\left(c/\lambda^{}\right)^{1/\alpha}\). Replacing \(\lambda\) with \((4\gamma m)^{-1}\) concludes the proof.
\end{proof}
Lemma \ref{le:RegBounds} and its proof are inspired by 
\citet[Lemma 14]{BachLeastSquaresJMLR2017}.    
\begin{lemma}\label{le:RegBounds} 
Suppose that Assumption As holds, and     \(\Lambda=(\log m)/(\gamma m)\). Then \begin{equation}\label{eq:regBiasBound}
||(I-\gamma \Sigma )^{m}\theta^{*}||^{2}_{\Sigma}\leq\left( \frac{2r+1}{2\gamma m} \right)^{2r+1}\rho^{*}.
\end{equation} If \(-1/2\le r \le1/2\) then\begin{equation}\label{eq:regVarBound}
 \Lambda^{2}\tilde\Lambda||(\Sigma+\Lambda)^{-1}(\Sigma+\tilde\Lambda)^{-1/2}\theta^*||^{2}_{\Sigma}\leq\rho^{*}\left(\frac{\log m}{\gamma m}\right)^{2r+1}.
\end{equation}If \(1/2\leq r\leq1\) then    \begin{equation}\label{eq:regVarBoundBis}
 \Lambda^{2}\tilde\Lambda||(\Sigma+\Lambda)^{-1}(\Sigma+\tilde\Lambda)^{-1/2}\theta^*||^{2}_{\Sigma}\leq\rho^{*}\left(\frac{\log m}{\gamma m}\right)^{2}\tilde\Lambda^{2r-1} .
\end{equation}    \comment{If \(1/2\leq r\leq 1\) then \begin{equation}\label{eq:regUpperB3}
 ||\theta_{\lambda}-\theta^{*}||^{2}_{\Sigma}\leq \lambda^{2}R^{4r-2}||v^{*}||^{2}.
\end{equation}}
\end{lemma}
\begin{proof}
 We have
\(
||(I-\gamma \Sigma )^{m}\theta^{*}||^{2}_{\Sigma}=\langle v^{*},Bv^{*}\rangle
\), where \(B=(I-\gamma \Sigma)^{2m}\Sigma^{2r+1}\). As \(e_{i}\) is an  eigenvalue of \(B\)  with corresponding eigenvalue    \((1-\gamma\mu_{i})^{2m}\mu_{i}^{2r+1}\), and  \(\gamma\mu_{i}\leq\gamma R^{2}\leq 1\) for \(i\in J\) by \eqref{eq:SigmaLeR^2} and \eqref{eq:GammaUpBound}, and since  \((1-z)^{2m}z^{2r+1}\leq((2r+1)/(2m))^{2r+1}\) for \(0\leq z\leq 1\), we have  \begin{displaymath}
B\preccurlyeq\left( \frac{2r+1}{2\gamma m} \right)^{2r+1}I.\end{displaymath}
This implies \eqref{eq:regBiasBound}. Moreover,  \begin{displaymath}
 ||(\Sigma+\Lambda)^{-1}(\Sigma+\tilde\Lambda)^{-1/2}\theta^*||^{2}_{\Sigma}=\langle v^{*},\Sigma^{2r+1}(\Sigma+\Lambda)^{-2}(\Sigma+\tilde\Lambda)^{-1}v^{*}\rangle.
\end{displaymath}If \(-1/2\le r \le1/2\) then\begin{displaymath}
\Sigma^{2r+1}(\Sigma+\Lambda)^{-2}(\Sigma+\tilde\Lambda)^{-1}\preccurlyeq(\Sigma+\Lambda)^{2r-1}(\Sigma+\tilde\Lambda)^{-1}\preccurlyeq\frac{\Lambda^{2r-1}}{\tilde\Lambda}I, \end{displaymath}where each inequality follows by a standard diagonalization argument. This implies \eqref{eq:regVarBound}. If \(1/2\leq r\leq1\) then \begin{displaymath}
\Sigma^{2r+1}(\Sigma+\Lambda)^{-2}(\Sigma+\tilde\Lambda)^{-1}\preccurlyeq(\Sigma+\tilde\Lambda)^{2r-2}\preccurlyeq\tilde\Lambda^{2r-2} I, \end{displaymath}
which yields  \eqref{eq:regVarBoundBis}. 
 \end{proof}
   
\subsection{Implications on Convergence} \label{sub:implications}  
Theorem~\ref{th:noiselessArAc} provides convergence bounds on \(E[\mathcal{R}(w_N)]\) under suitable assumptions.  
\begin{theorem}\label{th:noiselessArAc}
Suppose that \(m \ge 3\), Assumptions~A1, As and \eqref{eq:GammaUpBound} hold, and \(\Lambda = \log(m)/(\gamma m)\). 
If \(-1/2\le r \le1/2\) then \begin{equation}\label{eq:TotalBiasAtTArAc}
 E[\mathcal{R}(w_N)]\leq\frac{16}{(\gamma m)^{2r+1}}\rho^{*}\left(1+\gamma R^{2}(\log m)^{2r+1}\right).
\end{equation} 
If Assumption Ac holds and   \(1/2\leq r\leq1\) then \begin{equation}
\label{eq:TotalBiasAtTArAcBis}E[\mathcal{R}(w_N)]\leq\frac{16}{(\gamma m)^{2r+1}}\rho^{*}\left(1+\gamma R^{2} (\log m)^{2}(4c\gamma m)^{\frac{2r-1}{\alpha}}\right).
\end{equation}   

\end{theorem}
\begin{proof}
If  \(-1/2\leq r\leq 1/2\) then applying Theorem~\ref{th:noiseless} together with  \eqref{eq:regBiasBound} and  \eqref{eq:regVarBound} implies \eqref{eq:TotalBiasAtTArAc}. Likewise, if Assumption Ac holds and    \(1/2\leq r\leq1\) then using   Theorem \ref{th:noiseless},   \eqref{eq:regBiasBound},  \eqref{eq:regVarBoundBis} and    Lemma~\ref{le:BoundCapacity} implies \eqref{eq:TotalBiasAtTArAcBis}. \end{proof}
Theorem~\ref{th:noiselessArAc} implies immediately the following. Note that  \eqref{eq:GammaUpBound} holds under the conditions stated in Corollary \ref{cor:ArAcNoiselessAvg} because  \(2\log(m)\leq m\) for \(m\geq3\).        
\begin{corollary}\label{cor:ArAcNoiselessAvg}Suppose that Assumptions~A1  and As hold, and \(\gamma\) is a constant with  \(0<2\gamma R^{2} \leq1\). Set   \(\Lambda=\log(m)/(\gamma m)\). \begin{enumerate}
\item 
If \(-1/2\le r \le1/2\),  then there exists a constant \(\kappa_{1}\), depending only on \(\rho^{*}\) and \(\gamma\), such that for  \(m \ge 3\),\begin{equation}\label{eq:corNoiseless}
E[\mathcal{R}(w_N)]\leq\kappa_{1}\left(\frac{\log m}{ m}\right)^{2r+1}.
\end{equation}\item
If   \(1/2\leq r\leq1\) and Assumption Ac holds, then there exists a constant \(\kappa_{2}\),  depending only on \(c\), \(\rho^{*}\) and \(\gamma\), such that for  \(m \ge 3\), \begin{equation}\label{eq:CorNoiselessBis}
E[\mathcal{R}(w_N)]\leq\kappa_{2}(\log m)^{2}m^{\frac{2r-1}{\alpha}-2r-1}.
\end{equation}
\end{enumerate}\end{corollary}
An immediate consequence of Corollary~\ref{cor:ArAcNoiselessAvg} is that
   \(E[\mathcal{R}(w_N)]=O\left(m^{-2}\log^{2}m\right)\) when \(r=1/2\), while \(E[\mathcal{R}(w_N)]\)  converges strictly faster than   \(m^{-2}\) when \(r>1/2\) and \(\alpha>1\). Moreover,   \(E[\mathcal{R}(w_N)]=O(m^{-3+\epsilon})\) for any \(\epsilon>\alpha^{-1}\) when \(r=1\).   However, values of \(r>1\) do not  lead to improved bounds in Corollary \ref{cor:ArAcNoiselessAvg}. Such saturation in the source parameter is common   in the SGD literature \cite{BachDieuleveut2016,BachLeastSquaresJMLR2017}.
\subsection{Relation with previous work}\label{sub:comparison} 
We assume throughout this subsection  that \(\sigma=0\).
Under assumptions similar to ours, \citet[Theorem 2]{BachDieuleveut2016} implies a bound on the expected excess risk of averaged SGD with \(m\) time steps of order
\(m^{-2r-1}\) for \(-1/2\le r\le 0\) and
\(m^{-2r-1+2r/\alpha}\) for \(0\le r\le 1/2\), with a saturation effect for \(r\ge 1/2\). 
 The bounds on    \(E[\mathcal{R}(\bar w_{2m:3m})]\) implied by  Corollary \ref{cor:ArAcNoiselessAvg} match their bounds, up to a \(\log^{2}m\) factor, when \(-1/2\leq r\leq 0\), and are strictly sharper when \(r>0\).
Likewise, for \(0\le r\le 1/2\),   \citet[Theorem 2 and Lemma 14]{BachLeastSquaresJMLR2017} together with the bound
\(
\tr[\Sigma(\Sigma+\lambda)^{-1}]
= O(\lambda^{-1/\alpha})
\) for \(\alpha>1\)
\cite[Section 7.6.6]{bach2024learning}, imply that the expected excess risk of regularized averaged SGD with a constant step size is
\(O(\lambda^{2r+1}+  \lambda^{2r-1}m^{-2}(1+\lambda^{-1/\alpha}))\),
where \(\lambda\) is the regularization parameter and \(m\) is the number of time steps.
Up to a constant factor, this bound is minimized for  \(\lambda^{}=m^{-2\alpha/(2\alpha+1)}\), yielding a rate of  order \(m^{-2\alpha(2r+1)/(2\alpha+1)}\)  which is strictly higher than the bound  on \(E[\mathcal{R}(\bar w_{2m:3m})]\) implied by \eqref{eq:corNoiseless}. Finally, \citet{jun2019kernel} study a variant of KRR  in an RKHS setting, under capacity and source assumptions similar to ours and a boundedness assumption on the kernel and \(y\). When \(2\alpha r<-1\), their bounds on the expected excess risk   strictly improve upon  the preceding  ones and ours, and are of order  \(n^{-\alpha(2r+1)/(1+2r\alpha+\alpha)}\), where \(n\) is the number of training samples. Unlike \citet{jun2019kernel}, we do not assume that \(y\) is bounded.

Although we have focused above on their results for the noiseless case, the aforementioned papers mostly address the noisy case.
In contrast,      \citet{bach2020tight}  specifically study the noiseless  setting. They assume that  Assumption As holds with \(r\geq0\) and impose an additional regularity condition on \(x\). They provide a tight bound of the form \begin{equation}\label{eq:BertierBach}
\min_{t\in[0,m]}E[\mathcal{R}(w_{t})]\leq\hat\kappa m^{-2r-1},
\end{equation}
where \(\hat{\kappa}>0\) is a constant (recall that \((w_{t})_{0\leq t\leq m} \) is a standard SGD sequence in a noiseless setting).  Their regularity condition implies that   \(\Sigma^{1-2r}\) is trace-class, and thus  \(2r<1\) when   \(\mathcal{H}\) is infinite-dimensional. Therefore, for suitable choices of \(\gamma\) and \(\Lambda\), our bound \eqref{eq:corNoiseless} holds under their assumptions and matches the righthand side of \eqref{eq:BertierBach} up to a \(\log^{2}m\)  factor.  However, since \(2r<1\), the rate in \eqref{eq:BertierBach} decays strictly more slowly than \(m^{-2}\). Under related conditions, \citet{Flammarion2021last} provide a bound of order \(m^{-2r-1}\) on \(E[\mathcal{R}(w_{m})]\) for \(-1/2<r<1/2\).

In summary, when \(0<r<1/2\), our bounds for both averaged and non-averaged SGDIR are either strictly sharper than existing ones or hold under weaker assumptions. Moreover, whereas the preceding bounds saturate for \(r\geq 1/2\),  ours saturate only for \(r\geq 1\).    However, we do not establish convergence properties for the last iterate of SGDIR; this is left for future work.
\section{Comparison with ridge regression}\label{se:ridgeComparison}

In this section, we assume that \eqref{eq:R2Assumption} and the following conditional moments assumption hold. \begin{description}
\item[Assumption Acm.] There is a vector \(\theta^{*}\in \mathcal{H}\) and a constant  \(\underline\sigma\ge0\) with \(E[y - \langle x, \theta^{*} \rangle|x]=0\) and  \(E[(y - \langle x, \theta^{*} \rangle)^{2}|x]\geq\underline\sigma^{2}\).
\end{description}
Note that Assumption~Acm implies \eqref{eq:NormalEq}, and hence \(\theta^{*}\) minimizes \(L\).
Conversely, if there exists \(\theta^{*}\in\mathcal{H}\) such that
\(y-\langle x,\theta^{*}\rangle\) is mean-zero and independent of \(x\), then Assumption~Acm and \eqref{eq:SigmaAssumption} hold with
     \(\sigma ^{2}=\underline\sigma^{2}=E[(y - \langle x, \theta^{*} \rangle)^{2}]\).
If both Assumption~Acm and \eqref{eq:SigmaAssumption} hold and \(\Sigma\) is non-zero, then \(\underline\sigma\leq \sigma\).

Let \(\hat\theta_{\lambda}=\hat\theta_{\lambda,n}\) be the estimator of \(\theta^{*}\) obtained from a ridge regression using regularization parameter \(\lambda>0\) and  \(n\) independent copies \((x_{i},y_{i})\) of \((x,y)\), \(1\leq i\leq n\). Thus, \begin{displaymath}
\hat\theta_{\lambda}=\arg\min_{\theta\in\mathcal{H}}\sum^{n}_{i=1}(y_{i}-\langle x_{i},\theta\rangle)^{2}+\lambda||\theta||^{2}.
\end{displaymath}Let \(\hat\Sigma:=n^{-1}\sum^{n}_{i=1}x_{i}\otimes x_{i}\) be the empirical covariance operator. As \(E[x\otimes x]=\Sigma\), we have \(E[\hat\Sigma]=\Sigma\). Moreover,\begin{displaymath}
E[(x\otimes x-\Sigma)^{2}]=E[||x||^{2}x\otimes x]-\Sigma^{2}\preccurlyeq R^{2}\Sigma.
\end{displaymath}  By a standard calculation,  it follows that \(E[(\hat\Sigma-\Sigma)^{2}]\preccurlyeq (R^{2}/n)\Sigma \).  Hence  \begin{equation}
\label{eq:FirstTwoMoments}
E[\hat\Sigma]=\Sigma\text{ and }E[\hat\Sigma^{2}]\preccurlyeq\Sigma^{2}+ \frac{R^{2}}{n}\Sigma.
\end{equation} Furthermore, a straightforward adaptation of  \citet[Section 7.6.2]{bach2024learning} yields \begin{equation}\label{eq:RidgeBiasVar}
2E\left[\mathcal{R}(\hat\theta_{\lambda})\right]\geq\lambda^{2}E\left[||(\hat\Sigma+\lambda)^{-1}\theta^{*}||^{2}_{\Sigma}\right]+\frac{\underline\sigma^{2}}{n}E\left[\tr\left[(\hat\Sigma+\lambda)^{-2}\hat\Sigma\Sigma\right]\right].
\end{equation}
We now  derive lower bounds on each term on the righthand side of \eqref{eq:RidgeBiasVar}. 
\subsection{First term} 
We apply the equality \(E[\hat\Sigma]=\Sigma\) to derive a lower bound on the first term. When \(\mathcal{H}\) is one-dimensional, such a lower bound follows immediately from Jensen's inequality together with the convexity of \(z\mapsto (z+\lambda)^{-2}\) on \(\mathbb{R}_+\). In contrast, convexity on \(\mathbb{R}_+\) does not, in general, imply the corresponding operator inequality \cite[Example V.1.4]{bhatia2013matrix}. Lemma~\ref{le:OperatorCvxSq} circumvents this issue by establishing a Jensen-type inequality for the function \(z\mapsto (z+1)^{-2}\) in the operator setting.

 \begin{lemma}\label{le:OperatorCvxSq}Let \(U\) be a self-adjoint, positive semidefinite, square-integrable operator on \(\mathcal{H}\), and let \(V = E[U]\). Then \(V (V+I)^{-2}\preccurlyeq 4E[(U+I)^{-1} V (U+I)^{-1}]
\). 
\end{lemma}
Applying Lemma \ref{le:OperatorCvxSq} with \(U=\lambda^{-1}\hat\Sigma\) and noting that \(V=\lambda^{-1}\Sigma\) yields \begin{displaymath}
\Sigma (\Sigma+\lambda)^{-2}\preccurlyeq 4E\left[(\hat\Sigma+\lambda)^{-1} \Sigma (\hat\Sigma +\lambda)^{-1}\right].
\end{displaymath}
Consequently,\begin{equation}\label{eq:RidgeFirstTerlLowerBound}
|| (\Sigma+\lambda)^{-1}\theta^{*}||^{2}_{\Sigma}\leq4E\left[||(\hat\Sigma+\lambda)^{-1}\theta^{*}||^{2}_{\Sigma}\right].
\end{equation}

\subsection{Second term}We now obtain a lower bound on the second term by exploiting \eqref{eq:FirstTwoMoments}. Consider first  the case where \(\mathcal{H}\) is one-dimensional, and assume for simplicity  that \(\lambda=1\). The second term is proportional to \(E[f(\hat\Sigma)]\), where  \(f(z)=z(z+1)^{-2}\) for \(z\geq0\).    Inspired by  \citet{BP02}, we first bound \(f\) from below on \([0,\infty)\)  by  a quadratic function that matches \(f\) at the origin and is tangent to \(f\) at a second  point. 
 Specifically, for \(a,s\geq0\),  the equality\begin{displaymath}
(s+1)^{3}-(a+1)^{2}(3s+1-2a)=(a-s)^{2}(2a+s+3)
\end{displaymath} 
implies that \begin{equation}\label{eq:rationalInequality}
f(a)\geq(s+1)^{-3}[(3s+1)a-2a^{2}].
\end{equation}
Note  that \(f\) coincides at \(0\) with the quadratic function of \(a\) defined by the righthand side, and is tangent to it at \(s\).
Replacing \(a\) by \(\hat \Sigma\), taking expectations and using  \eqref{eq:FirstTwoMoments} yields \begin{equation}
E[f(\hat\Sigma)]\ge(s+1)^{-3}\Sigma\left[(3s+1)-2\left(\Sigma+ \frac{R^{2}}{n}\right)\right].  
\end{equation}
The righthand side is maximized when \(s=\Sigma+R^{2}/n\), implying that \begin{displaymath}
E[f(\hat\Sigma)]\ge\left(\Sigma+ \frac{R^{2}}{n}+1\right)^{-2}\Sigma.
\end{displaymath}

To generalize this approach to arbitrary dimensions, we seek an operator analogue of \eqref{eq:rationalInequality} in which \(a\) is replaced by \(\hat\Sigma\) and \(s\) by \(\Sigma+R^{2}/n\),  up to an appropriate scaling by \(\lambda\).
 This is achieved via Lemma~\ref{le:TraceGAB} below, which lifts certain scalar inequalities to trace inequalities for operators under suitable conditions. 
Let \(A\) be a self-adjoint positive semidefinite operator on \(\mathcal{H}\) that is diagonalizable in an orthonormal eigenbasis \((u_i)_{i\in J}\) of \(\mathcal{H}\). Let \(a_i\) denote the eigenvalue of \(A\) associated with \(u_i\), so that
\(Au_{i}=a_{i}u_{i}\) for \(i\in J\).
As \(0\preccurlyeq A\), we have \(0\leq a_{i}\leq||A||\) for \(i\in J\).
We say that a rational function \(\phi\) is well defined on \([0,||A||]\) if it admits a representation \(\phi=P/Q\), where \(P\) and \(Q\) are real polynomials and \(Q(z)>0\) for  \(z\in[0,||A||]\). In this case, we define the operator \(\phi(A):=P(A)Q(A)^{-1}\) on  \(\mathcal{H}\).  As  \(\phi(A)u_{i}=\phi(a_{i})u_{i}\) for \(i\in J\),  the operator \(\phi_{}(A)\) does not depend on the choice of \(P\) and \(Q\).          
\begin{lemma}\label{le:TraceGAB}Let \(A\) and \(B\) be self-adjoint positive semidefinite operators on \(\mathcal{H}\), each diagonalizable in an orthonormal eigenbasis. Let  \(k\) be a positive integer, and for \(1\leq j\leq k\), let \(\phi_j\) and \(\nu_j\) be rational functions that are well defined on \([0,\|A\|]\) and \([0,\|B\|]\), respectively.     Assume that \(g(a,b)\ge 0\)  for \(a\in[0,||A||]\) and \(b\in[0,||B||]\), where  \(g(a,b):=\sum^{k}_{j=1}\phi_{j}(a)\nu_{j}(b)\), and that the operator   \(g(A,B):=\sum^{k}_{j=1}\phi_{j}(A)\nu_{j}(B)\)   is trace-class. Then \(\tr[g(A,B)]\ge0\).
\end{lemma}
\begin{proof}
Define  \(u_{i}\) and   \(a_{i}\) for \(i\in J\) as above, and let \(v\) be an eigenvector of \(B\) with eigenvalue \(b\). Thus   \(\phi_{j}(A)u_{i}=\phi_{j}(a_{i})u_{i}\) and  \(\nu_{j}(B)v=\nu_{j}(b)v\) for \(i\in J\)  and \(1\leq j\leq k\). 
Moreover,\begin{displaymath}
\langle v,g(A,B)v\rangle=\langle v,\sum^{k}_{j=1}\phi_{j}(A)\nu_{j}(B)v\rangle=\sum^{k}_{j=1}\nu_{j}(b)\langle v,\phi_{j}(A)v\rangle =\langle v,g(A,b)v\rangle,
\end{displaymath}
where \(g(A,b):=\sum^{k}_{j=1}\nu_{j}(b)\phi_{j}(A)\) by convention.    Because   \(g(A,b)u_{i}=g(a_{i},b)u_{i}\) for \(i\in J\) and  \(g(a_{i},b)\geq0\),   the operator \(g(A,b)\)  is self-adjoint and positive semidefinite. Thus,   \(\langle v,g(A,B)v\rangle\ge0\). As \(g(A,B)\) is trace-class, we have \(\tr[g(A,B)]=\sum_{i\in J}\langle v_{i},g(A,B)v_{i}\rangle\),  where \((v_i)_{i\in J}\) is an orthonormal eigenbasis of \(\mathcal{H}\) for \(B\), 
which yields  \(\tr[g(A,B)]\ge0\).
\end{proof}
\comment{\begin{lemma}Let \(g(a,b)=b\phi(a)+\nu(b)+(a-b)^{^{2}}\zeta(b)\). If \(g(a,b)\ge0\) for \(a,b\geq0\), then \(\tr\left[g(A,B))\ge0\) for any self-adjoint positive semidefinite operators  \(A\) and \(B\).
\end{lemma}
\begin{proof}Let \(v\) be an eigenvector of \(B\) with eigenvalue \(b\). Because \(A\) and \(B\) are self-adjoint,\begin{displaymath}
\langle v,g(A,B)v\rangle=\langle v,B\phi(A)v\rangle+\langle v,\nu (B)v\rangle+\langle v,(A-B)^{^{2}}\zeta(B)v\rangle=b\langle v,\phi(A)v\rangle+\nu (b)||v||^{2}+\zeta(b)||(A-bI)v||^{2}=\langle v,g(A,b)v\rangle,
\end{displaymath}
The eigenvalues of \(g(A,b)\) are  \(g(a_{i},b)\), where \((a_{i})_{i\in J}\) are the eigenvalues of \(A\). Since \(0\preccurlyeq A\) and  \(0\preccurlyeq B\), we have \(b\geq0\)  and \(a_{i}\ge0\) for \(i\in J\). Thus, the eigenvalues of \(g(A,b)\) are
nonnegative, which implies that \(\langle v,g(A,B)v\rangle\ge0\). Since this inequality holds for any eigenvector  \(v\)  of \(B\), it follows that  \(\tr\left[g(A,B))\ge0\).
\end{proof}
}
\begin{lemma}\label{le:lowerBoundTr}For \(\lambda>0\),  \begin{equation}\label{eq:lowerBoundTr}
E\left[\tr\left(\left(\hat\Sigma+\lambda\right)^{-2}\hat\Sigma\Sigma\right)\right]\ge\tr\left[\Sigma^{2}\left(\Sigma+\frac{R^{2}}{n}+\lambda\right)^{-2}\right].
\end{equation} 
\end{lemma}\begin{proof}
As \(x\otimes x\) is a trace-class self-adjoint positive semidefinite operator, so is \(\hat\Sigma\). Hence,   \(\hat\Sigma\) is diagonalizable in an orthonormal eigenbasis. For \(a,b\geq0\),  
replacing \(a\) with \(\lambda ^{-1}a\) and  \(s\) with \(\lambda ^{-1}(b+R^{2}/n)\) in \eqref{eq:rationalInequality} yields   \begin{displaymath}
\frac{ab}{(a+\lambda)^{2}}\geq\frac{(3a(b+R^{2}/n)+\lambda a-2a^{2})b}{(b+R^{2}/n+\lambda)^{3}}.
\end{displaymath} 
 Applying Lemma \ref{le:TraceGAB} with \(A=\hat\Sigma\) and \(B=\Sigma\) then shows that \begin{displaymath}
\tr\left[(\hat\Sigma+\lambda)^{-2}\hat\Sigma\Sigma\right]\geq\tr\left[\left(3\hat\Sigma\left(\Sigma+\frac{R^{2}}{n}\right)+\lambda\hat\Sigma-2\hat\Sigma^{2}\right)\Sigma\left(\Sigma+\frac{R^{2}}{n}+\lambda\right)^{-3}\right].
\end{displaymath}
Taking expectations and using  \eqref{eq:FirstTwoMoments} implies \eqref{eq:lowerBoundTr} after simplifications. 
\end{proof}
\subsection{Combining bounds}

Combining  \eqref{eq:RidgeBiasVar}, \eqref{eq:RidgeFirstTerlLowerBound} and Lemma \ref{le:lowerBoundTr} yields the following. \begin{theorem}\label{th:Ridge}If  \eqref{eq:R2Assumption} and Assumption Acm hold, then \begin{equation*}
E\left[\mathcal{R}(\hat\theta_{\lambda})\right]\geq\frac{\lambda^{2}}{8}|| (\Sigma+\lambda)^{-1}\theta^{*}||^{2}_{\Sigma}+\frac{\underline\sigma^{2}}{2n}\tr\left[\Sigma^{2}\left(\Sigma+\frac{R^{2}}{n}+\lambda\right)^{-2}\right].
\end{equation*} 
\end{theorem}
Proposition~\ref{pr:RidgeVsSGD} compares \(E[\mathcal{R}(\hat\theta_{\lambda})]\) and \(E[\mathcal{R}(\bar\theta_{2m:3m})]\), taking \(m=n/3\) so that both estimators use roughly the same number of samples of \((x,y)\).
\begin{proposition}\label{pr:RidgeVsSGD}
Suppose that Assumptions~A1 and Acm hold, and  \(n\geq9\). Set \(m=n/3\) and let  \(\lambda>0\). \begin{enumerate}
\item If \(\sigma=\underline\sigma=0\), set 
 \(\gamma=(2R^{2})^{-1}\) and \(\Lambda=(\log m)/(\gamma m)\).
Then\begin{equation}\label{eq:RidgeVsSGDNoSigma}
E\left[\mathcal{R}(\bar\theta_{2m:3m})\right]\leq c_{1}\left( \frac{R^{4}\log ^{2}n}{ \lambda^{2} n^{2}}+1 \right)E\left[\mathcal{R}(\hat\theta_{\lambda})\right],
\end{equation}
where \(c_{1}\) is an absolute constant.
\item 
If \(\underline\sigma>0\), set 
 \(\gamma=(2R^{2}+\lambda   m/\sqrt{\log m})^{-1}\) and \(\Lambda=(\log m)/(\gamma m)\).
Then \begin{equation}\label{eq:RidgeVsSGDSigma}
E\left[\mathcal{R}(\bar\theta_{2m:3m})\right]\leq c_{2}\log n\left(\frac{R^{4}\log n}{\lambda ^{2}n^{2}}+\frac{\sigma^{2}}{\underline\sigma^{2}}\right)E\left[\mathcal{R}(\hat\theta_{\lambda})\right],
\end{equation}where \(c_{2}\) is an absolute constant. \end{enumerate}
\end{proposition}
\begin{proof}
A standard diagonalization argument shows that, for \(\Lambda>0\),  \begin{equation}\label{eq:diagRidge1}
||(\Sigma+\Lambda)^{-1}\theta^{*}||^{2}_{\Sigma}\le \left(1+ \frac{\lambda^{2}}{\Lambda^{2}} \right)||(\Sigma+\lambda)^{-1}\theta^{*}||_{\Sigma},
\end{equation}  
and, for \(\lambda',\lambda''>0\), \begin{equation}\label{eq:diagRidge2}
\tr\left[\Sigma^{2}\left(\Sigma+\lambda'\right)^{-2}\right]\leq\left(1+ \frac{{\lambda''}^{2}}{{\lambda'}^{2}}\right) \tr\left[\Sigma^{2}\left(\Sigma+\lambda''\right)^{-2}\right].
\end{equation}
If  \(\sigma=\underline{\sigma}=0\), then \eqref{eq:RidgeVsSGDNoSigma} follows from Proposition~\ref{pr:SimpleBound}, Theorem~\ref{th:Ridge}, and \eqref{eq:diagRidge1}. Now suppose that \(\underline{\sigma}>0\). Combining   Proposition~\ref{pr:SimpleBound}, Theorem~\ref{th:Ridge}, and \eqref{eq:diagRidge1} with \eqref{eq:diagRidge2} applied with \(\lambda'=(4\gamma m)^{-1}\) and \(\lambda''=\lambda+R^{2}/n\),   yields \eqref{eq:RidgeVsSGDSigma} after some calculations.\end{proof}
Proposition~\ref{pr:RidgeVsSGD} implies that, if \(\lambda\) is of order \(1/n\) or larger, then \(E[\mathcal{R}(\bar\theta_{2m:3m})]\) is bounded above by   \(E[\mathcal{R}(\hat\theta_{\lambda})]\) up to a \(\log^{2}n\) factor, both when  \(\sigma=\underline\sigma=0\) and when \(\underline\sigma>0\), for a suitable choice of \(\gamma\). In the former case, the gap reduces to a constant factor when  \(\lambda\) is of order   \((\log n)/n\) or larger. In the latter case, it reduces to a  \(\log n\)  factor when  \(\lambda\) is of order   \(\sqrt{\log n}/n\) or larger.
\section{Lower bound}\label{se:lowerBound}
Let \(\mathcal{H}\) be an infinite-dimensional separable Hilbert space over \(\mathbb{R}\), and let \(\Sigma\) be a positive semidefinite  trace-class operator on \(\mathcal{H}\), with eigenvalues  \(\mu_{i}\), \(i\geq1\), sorted in nonincreasing order. Assume that there exists \(\rho\in(0,1)\) such that \(\rho\mu_{i}\leq\mu_{i+1}\) for \(i\geq1\). Let  \(\theta^{*}\in \mathcal{H}\)  and let  \((x,y)\) be a square-integrable random vector in  \(\mathcal{H}\times\mathbb{R}\) that  satisfy Assumptions A1 and As, with  \(E[x\otimes x]=\Sigma\) and \(\sigma=0\). Fix   \(r\geq-1/2\), and let \(n\geq2\)  be an integer such that \(n\mu_{1}>\tr[\Sigma]\).  

We consider the problem of finding an estimator \(\hat\theta=\hat\theta((x_{1},y_{1}),\dots,(x_{n},y_{n}))\) that approximately  minimizes \(L\) based on \(n\)   copies \((x_{1},y_{1}),\dots,(x_{n},y_{n})\) of \((x,y)\). Theorem \ref{th:lower} below provides  a minimax lower bound on the expected excess risk of \(\hat\theta\). The intuition behind our approach is to construct two pairs    \((x,y^{(0)})\) and \((x,y^{(1)})\) that satisfy the above assumptions and generate identical training samples of size \(n\)   with constant probability. 
This makes it difficult for any learning algorithm to distinguish between the two distributions.  Minimax lower bounds are often established in a noisy setting (e.g., \citet[Section 15.1]{bach2024learning}). 

We use   a one-hot distribution for \(x\). Such distributions have previously been considered in the SGD literature  
\cite{jain2018accelerating,kakade2021benefits}, and our lower bound approach is inspired by \citet{jain2018accelerating}. 
\begin{theorem}\label{th:lower}
  For any  measurable function    \(\hat\theta:(\mathcal{H}\times\mathbb{R})^{n}\rightarrow\mathcal{H}\), there is  \(\theta^{*}\in\mathcal{H}\)  satisfying Assumption As with parameter \(r\) and \(\rho^{*}=1\),  and a random pair \((x,y)\) with \(E[x\otimes x]=\Sigma\) satisfying Assumption A1,  with \(R^{2}=\tr[\Sigma]\) and \(\sigma=0\), such that
\begin{equation}\label{eq:LowerBoundUniv}
\frac{1}{4}\left(\frac{\tr[\Sigma]\rho}{n}\right)^{2r+1}\leq E\left[||\hat\theta((x_{1},y_{1}),\dots,(x_{n},y_{n}))-\theta^{*}||_{\Sigma}^{2}\right],
\end{equation}where  \((x_{t},y_{t})\), \(1\leq t\leq n\), are independent copies of \((x,y)\). \end{theorem}
\begin{proof}We first build two noiseless pairs     \((x,y^{(0)})\) and \((x,y^{(1)})\) that satisfy Assumptions A1 and As, with \(E[x\otimes x]=\Sigma\).  Consider an orthonormal  eigenbasis \((e_{i})_{i\geq1}\) of \(\mathcal{H}\) with \(\Sigma e_{i}=\mu_{i} e_{i}\) for \(i\geq1\).    Let \(x\) be a random vector such that   \(\Pr\left[x=\sqrt{\tr[\Sigma]}e_{i}\right]=p_{i}\) for \(i\geq1\), where   \(p_{i}=\mu_{i}/\tr[\Sigma]\). Thus\begin{equation*}
E[x\otimes x]=\sum^{\infty}_{i=1}p_{i}\tr[\Sigma]e_{i}\otimes e_{i}=\sum^{\infty}_{i=1}\mu_{i}e_{i}\otimes e_{i}=\Sigma.
\end{equation*}Let \(j=\min\{i\geq1:np_{i}\leq1\}\). Since  \(n\mu_{1}>\tr[\Sigma]\), we have \(j>1\). We  choose \(\theta^{*}\) among two possible values, \({\theta^{*}}^{(0)}\) and \({\theta^{*}}^{(1)}\),  where \({\theta^{*}}^{(l)}=(-1)^{l}\mu _{j}^{r}e_{j}\) for \(l\in\{0,1\}\), and  set \(y^{(l)}=\langle{\theta^{*}}^{(l)},x\rangle\). 
As   \(||x||^{2}=\tr[\Sigma]\) almost surely,  Assumption A1 holds for both \((x,y^{(0)})\) and  \((x,y^{(1)})\),  with \(R^{2}=\tr[\Sigma]\) and  \(\sigma=0\). Moreover, for \(l\in\{0,1\}\), Assumption As holds for  \({\theta^{*}}^{(l)}\) with \({v^{*}}^{(l)}=(-1)^{l}e_{j}\)   and  \(\rho^{*}=1\). 

For \(l\in\{0,1\}\), let   \((x_{t},y^{(l)}_{t})\), \(0\leq t\leq n-1\),  be  \(n\) independent copies of \((x,y^{(l)})\), and  let \(\hat \theta^{(l)}=\hat\theta((x_{1},y_{1}^{(l)}),\dots,(x_{n},y_{n}^{(l)}))\).   
 Consider the event \begin{displaymath}
\mathcal{A}=\{\langle x_{t},e_{j}\rangle=0 \text{ for  }0\leq t\leq n-1\}.
\end{displaymath}Since \(y_{t}^{(l)}=\langle{\theta^{*}}^{(l)},x_{t}\rangle\),  on the event \(\mathcal{A}\), we have  \(y_{t}^{(l)}=0\)  for \(0\leq t\leq n-1\) and  \(l\in\{0,1\}\),     which implies that  \(\hat \theta^{(0)}= \hat \theta^{(1)}\).      Since \(||\theta||_{\Sigma}^{2}\geq\mu_{j}\langle\theta,e_{j}\rangle^{2}\) for \(\theta\in\mathcal{H}\), it follows that, on the event \(\mathcal{A}\), \begin{displaymath}
\sum^{1}_{l=0}||\hat\theta^{(l)}-{\theta^{*}}^{(l)}||_{\Sigma}^{2}\ge\mu_{j}\sum^{1}_{l=0}\langle\hat\theta^{(l)}-{\theta^{*}}^{(l)},e_{j}\rangle^{2}\ge2\mu _{j}^{2r+1},
\end{displaymath}
where the second equation follows from the inequality \((z'-z)^{2}+(z'+z)^{2}\geq2z^{2}\) applied with \(z'=\langle\hat\theta^{(0)},e_{j}\rangle\) and \(z=\mu _{j}^{r}\). Moreover, it follows from the definition of \(j\) that  \(\rho\leq\rho np_{j-1}\leq n p_{j}\), where the second equation follows from the inequality  \(\rho\mu_{j-1}\leq\mu_{j}\).  Moreover, as  \( np_{j}\leq1\), \begin{displaymath}
\Pr[\mathcal{A}]=(1-p_{j})^n\ge(1-\frac{1}{n})^{n}\geq\frac{1}{4}, 
\end{displaymath}where the last equation is valid for all \(n\geq2\). Hence, \begin{displaymath}
\sum^{1}_{l=0}E[||\hat\theta^{(l)}-{\theta^{*}}^{(l)}||_{\Sigma}^{2}]\ge\frac{1}{2}\mu _{j}^{2r+1}\ge\frac{1}{2}\left(\frac{\tr[\Sigma]\rho}{n}\right)^{2r+1}, \end{displaymath}
where the last equation follows from the inequality \(\rho\leq np_{j}\). Thus, \eqref{eq:LowerBoundUniv} is satisfied either for \(\theta^{*}={\theta^{*}}^{(0)}\) and \((x,y)=(x,y^{(0)})\), or for  \(\theta^{*}={\theta^{*}}^{(1)}\) and \((x,y)=(x,y^{(1)})\).  
\end{proof}
In a noisy setting, for \(\alpha\geq1\) and \(r>-1/2\), the expected excess risk of any algorithm based on \(n\) training samples admits an information-theoretic lower bound of order \(n^{-\alpha(2r+1)/(1+2\alpha r+\alpha)}\) (e.g., \citet{vito2007optimal} and \cite[Section 15.1.5]{bach2024learning}). When \(2\alpha r<-1\), this lower bound is asymptotically smaller than the lower bound in  Theorem \ref{th:lower}.

Note that \(y^{(0)}\) and \(y^{(1)}\) are not bounded in the proof of Theorem \ref{th:lower} when \(r<0\). Therefore, in settings where a boundedness condition on \(y\) is imposed, the lower bound in  Theorem \ref{th:lower} may not apply. In particular, for \(\alpha\geq 1\) and \(-1/2<r\leq  0\),      \citet{jun2019kernel} obtain, in a noisy setting,  an upper bound of order  \(n^{-\alpha(2r+1)/(1+2r\alpha+\alpha)}\)  on the expected excess risk.  When \(2\alpha r<-1\), this upper bound     is asymptotically smaller than the lower bound in Theorem~\ref{th:lower}. A similar observation applies to \citet{BachPillaudNIPS2018}. Finally, the lower bound in Theorem \ref{th:lower} matches the upper bound of order \(n^{-2r-1}\) shown in a noisy setting by \citet[Corollary 2]{BachDieuleveut2016} when \(2\alpha r<-1\), and, up to a polylogarithmic factor, the upper bound in \eqref{eq:corNoiseless} when \(-1/2\le r \le1/2\).
\section{Numerical experiments}\label{se:numer}
The numerical experiments were conducted on a laptop with a 1 GHz Intel processor and 8 GB of RAM. The code was written in Python. We implemented the following estimators, each evaluated over \(5\times10^{5}/n\) independent runs, with each run using \(n\) random training samples. Standard tail-averaged SGD averages the last \(n/2\) iterates of standard SGD, while tail-averaged SGDIR computes \(\bar\theta_{n/2:n+1}\) with \(m=n/4\). In the noiseless setting of Section~\ref{sub:synthetic}, we additionally implemented last-iterate SGD, last-iterate SGDIR given by \(\theta_{n+1}\) with \(m=n/2\), and ridge regression. For all SGDIR estimators, we set \(\Lambda=\log(m)/(2\gamma m)\). In Section~\ref{sub:realData}, we also implemented KRR. For both ridge regression and KRR, the regularization parameter was selected by \(5\)-fold cross-validation using the \(n\) training samples.
The test error of a model is defined as \(E[(\hat y-y)^{2}]\), where \(\hat y\) is the response value predicted by the model on input \(x\).

\subsection{Synthetic experiments}\label{sub:synthetic}
We assume that \(d=2000\),  that \(x\) is a centered \(d\)-dimensional Gaussian vector with covariance matrix \(\Sigma=\text{diag}(i^{-\alpha})\), \(1\leq i\leq d\), where \(\alpha>0\), and  that \(y=x^{T}\theta^{*}\), where \(\theta^{*}\in\mathbb{R}^{d}\) is deterministic. Similar distributions have been used in both noiseless \cite{Flammarion2021last} and noisy settings \cite{BachLeastSquaresJMLR2017,jain2018parallelizing,kakade2023benign}.   A standard calculation shows that  \eqref{eq:R2Assumption} holds with \(R^{2}=2\mu_{1}+\tr(\Sigma)\). We use the step size  \(\gamma=1/R^{2}\) for both SGD and SGDIR.  Fig. \ref{fig:noiseless} reports the average test errors of several estimators as functions of \(n\). The ridge regression test error is not reported for \(n\geq d\), since it is essentially \(0\)  when \(\lambda=0\).  For \(n\leq d\), ridge regression outperforms SGDIR and the optimal \(\lambda\) is of order \(10^{-6}\) or smaller, consistent with the fact that the bound \eqref{eq:RidgeVsSGDNoSigma}  becomes loose for such small values of \(\lambda\).   The slopes obtained by regressing the logarithm of the test error of  tail-averaged SGDIR  on  \(\log(n)\) in panels a), b) and c) in Fig. \ref{fig:noiseless} are respectively \(-1.75\), \(-1.98\) and \(-2.33\). The corresponding  slopes implied by Corollary~\ref{cor:ArAcNoiselessAvg} and the discussion at the beginning of Section \ref{sub:Assumptions} are   \(-1.8\), \(-2.28\) and \(-2.25\). 

\begin{figure}[t]
\centering

\begin{subfigure}[t]{0.32\textwidth}
\centering
\input{alpha=1,25s=1}
\caption{$\mu_i=i^{-1.25},\theta^{*}_i=i^{-1}$}
\end{subfigure}
\hfill
\begin{subfigure}[t]{0.32\textwidth}
\centering
\input{alpha=1,25s=2}
\caption{$\mu_i=i^{-1.25},\theta^{*}_i=i^{-2}$}
\end{subfigure}
\hfill
\begin{subfigure}[t]{0.32\textwidth}
\centering
\input{alpha=2s=2}\caption{$\mu_i=i^{-2},\theta^{*}_i=i^{-2}$}
\end{subfigure}

\vspace{1.5em}

\begingroup
\centering
\scriptsize
\setlength{\tabcolsep}{8pt}      
\renewcommand{\arraystretch}{1.3} 

\newcommand{\legendicon}[3]{
  \begin{tikzpicture}[baseline=0.9ex]%
    \begin{axis}[%
      width=22pt, height=12pt, 
      scale only axis,          
      hide axis,                
      xmin=0, xmax=1, ymin=0, ymax=1 
    ]%
    \addplot+[draw=#2, color=#2, #1, no markers, domain=0:1] {0.5};     
    \addplot+[draw=#2, color=#2, fill=#2, #1, only marks, mark options={draw=#2, fill=#2}] coordinates {(0.5,0.5)}; 
    \end{axis}%
  \end{tikzpicture}~#3%
}

\begin{tabular}{|lll|}
\hline
\legendicon{mark=*}{blue}{Last-iterate SGD} &
\legendicon{mark=square*}{red}{Last-iterate SGDIR} &
\legendicon{mark=triangle*}{brown}{Tail-averaged SGD} \\
\legendicon{mark=diamond*}{black}{Tail-averaged SGDIR} &
\legendicon{mark=pentagon*}{violet}{Ridge regression} & \\
\hline
\end{tabular}
\endgroup

\vspace{0.5em}
\caption{Test error versus $n$ for in a noiseless setting.}
\label{fig:noiseless}
\end{figure}
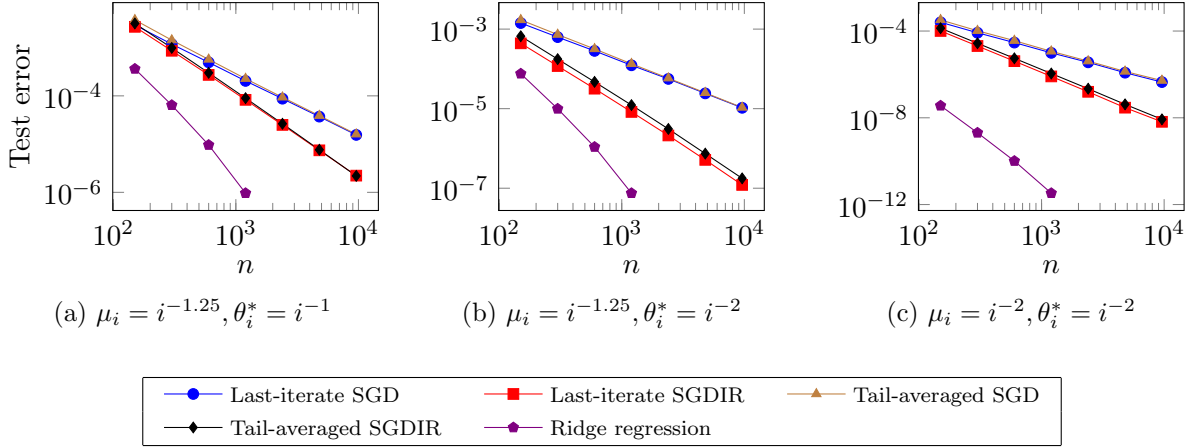

\subsection{Real data}\label{sub:realData}
  We use the datasets ``house\_8L'', ``elevators'' and ``year\_prediction\_msd'' (MSD) downloaded from OpenML.\comment{The first two are used in \citet{rosacso2025efficient}, and the last in}  Each dataset is divided into a training and a testing dataset, and the training dataset is normalized. We apply the exponential  kernel, with bandwidth selected by the median heuristic based on \(10^{4}\)  randomly chosen training pairs.
We implemented both SGD and SGDIR using an adaptation of the algorithm described in \cite[Section~7.4.5]{bach2024learning}, with  step size  \(\gamma=1\).
For KRR, we used the standard algorithm  described in \cite[Eq.~7.7]{bach2024learning}. For each run of each algorithm, the test error is estimated by averaging the squared difference between \(y\) and its fitted value over \(10^{3}\) randomly chosen test samples.  Table~\ref{tab:kernel_comparison} reports the running times of tail-averaged SGD and Ridge regression, together with the quantity \(\lambda n\), where \(\lambda\)  is the average regularization parameter for KRR. The running time of  SGDIR differs from that of tail-averaged SGD by at most \(20\%\) and is therefore omitted for  readability.  Because of vectorization, the running times of SGD and KRR increase little with the number of attributes \(d'\). Moreover, KRR is faster than SGD for small values of \(n\), but slower for large values. However, for the MSD dataset with \(n=3\times10^4\), tail-averaged SGDIR achieves an average test error of \(84.07\) over \(100\) independent runs and takes \(260\) seconds per run, whereas KRR results in a memory overflow due to its \(O(n^2)\) space requirement.
\begin{table}[ht]
\centering
\small
\begin{tabular}{rccccccccc}
\hline
$n$ &
\multicolumn{3}{c}{house\_8L ($d'=8$)} &
\multicolumn{3}{c}{elevators ($d'=18$)} &
\multicolumn{3}{c}{MSD ($d'=90$)} \\
 & SGD & KRR & $\lambda n$
 & SGD & KRR & $\lambda n$
 & SGD & KRR & $\lambda n$ \\
\hline
150  & 0.0016 & 0.00086 & 0.09
     & 0.0014 & 0.00079 & $3.5\times10^{-5}$
     & 0.0019 & 0.00081 & 2.1 \\

300  & 0.0039 & 0.0042 & 0.12
     & 0.0034 & 0.0034 & $1\times10^{-6}$
     & 0.0044 & 0.0034 & 0.12 \\

600  & 0.015 & 0.019 & 0.16
     & 0.014 & 0.018 & $1\times10^{-6}$
     & 0.016 & 0.019 & 0.11 \\

1200 & 0.048 & 0.084 & 0.19
     & 0.046 & 0.092 & $1\times10^{-6}$
     & 0.053 & 0.083 & 0.17 \\

2400 & 0.19 & 0.42 & 0.23
     & 0.20 & 0.42 & $1\times10^{-6}$
     & 0.17 & 0.41 & 0.23 \\
\hline
\end{tabular}
\caption{Running time (in seconds) of a single run of tail-averaged SGD (representative of both SGD-like methods) and KRR, together with the corresponding values of \(\lambda n\). The reported running times exclude cross-validation and test error evaluation.}
\label{tab:kernel_comparison}
\end{table}
 Fig. \ref{fig:datasets} provides the average test error in terms of \(n\).\begin{figure}[ht]
\centering
\begin{subfigure}[t]{0.32\textwidth}
\centering
\input{house_8LExponentialavgTesterror_vs_n.tex}
\end{subfigure}
\hfill
\begin{subfigure}[t]{0.32\textwidth}
\centering
\input{elevatorsExponentialavgTesterror_vs_n.tex}
\end{subfigure}
\hfill
\begin{subfigure}[t]{0.32\textwidth}
\centering
\input{year_prediction_msdExponentialavgTesterror_vs_n.tex}
\end{subfigure}

\caption{Average test error versus $n$ for the house\_8L, elevators and MSD  datasets, from left to right.}
\label{fig:datasets}
\end{figure}
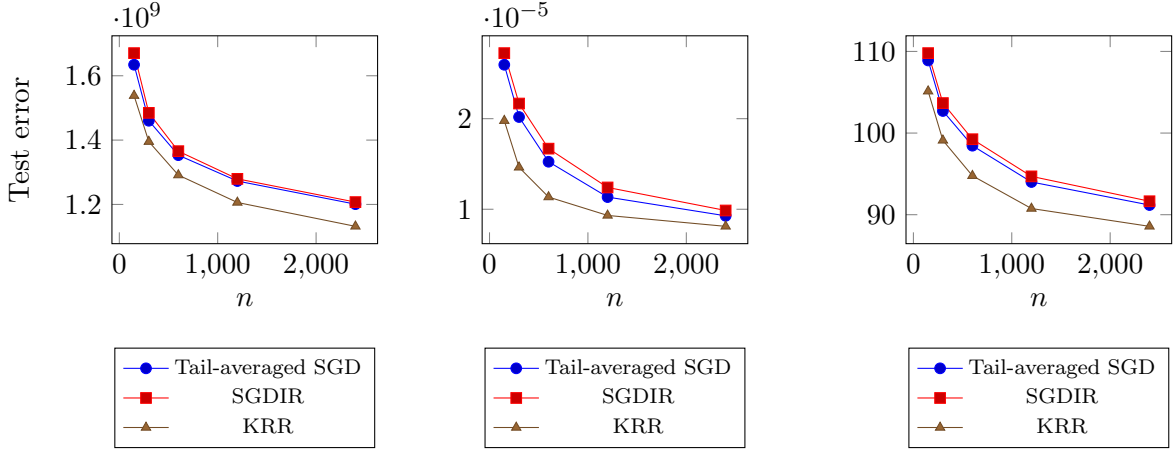
Fig. \ref{fig:comparison} provides an estimate of \begin{displaymath}R(n,\tilde n):=
\frac{E[L(\bar\theta_{n/2:n+1})]-E[L(\theta_{\text{SGD},\tilde n})]}{E[L(\hat\theta_{\lambda^{*},n})]-E[L(\theta_{\text{SGD},\tilde n})]},
\end{displaymath}
where \(n\) ranges over the values used in Fig. \ref{fig:datasets}, \(\lambda^{*}\) is  an optimal regularization parameter for KRR, and   \(\theta_{\text{SGD},\tilde n}\) is the tail-averaged estimator based on  \(\tilde n=14000\) training samples. We estimate \(E[L(\theta_{\text{SGD},\tilde n})]\) using \(100\) independent runs. Since  \(E[L(\theta_{\text{SGD},\tilde n})]\geq L(\theta^{*})\)  and provided that
 \(R(n,\tilde n)\ge1\) and the denominator in the definition of \(R(n,\tilde n)\) is positive, conditions that we verify numerically,  a simple calculation shows that \(E[\mathcal{R}(\bar\theta_{n/2:n+1})]\leq R(n,\tilde n)E[\mathcal{R}(\hat\theta_{\lambda^{*},n})]\). All values of \(R(n,\tilde n)\) reported in Fig. \ref{fig:comparison}  are   at most \(3\),   consistent with the \(O(\log^{2} n)\) upper bound implied by Proposition~\ref{pr:RidgeVsSGD} under suitable conditions. However,  Table \ref{tab:kernel_comparison} suggests that, for the elevators dataset, the requirement in the discussion following Proposition~\ref{pr:RidgeVsSGD}  that \(\lambda^{*}\) be of order \(1/n\) or larger is not satisfied in practice. 
\begin{figure}[ht]
\centering
\begin{tikzpicture}
\begin{axis}[
    width=0.5\textwidth,
    height=0.32\textwidth,
    xlabel={$n$},
    ylabel={$R(n,\tilde n)$},
    font=\scriptsize,
    legend pos=south east,]

\addplot+[
    mark=*
] coordinates {
    (150,1.2794108186411548)
    (300,1.2713830814560634)
    (600,1.3286548638461657)
    (1200,1.520739218299107)
    (2400,2.128817357174885)
};
\addlegendentry{House\_8L}

\addplot+[
    mark=square*
] coordinates {
    (150,1.5751745259325212)
    (300,1.9001403537615653)
    (600,2.1748122996858212)
    (1200,2.227139454314606)
    (2400,2.334626788218764)
};
\addlegendentry{Elevators}

\addplot+[
    mark=triangle*
] coordinates {
    (150,1.2272499583162566)
    (300,1.3144042817237602)
    (600,1.4392825386433947)
    (1200,1.637471494865933)
    (2400,1.7692716512060167)
};
\addlegendentry{MSD}
\end{axis}
\end{tikzpicture}
\caption{Upper bound on \({E[\mathcal{R}(\bar\theta_{n/2:n+1})]}/{E[\mathcal{R}(\hat\theta_{\lambda,n})]}\), with \(\tilde n=14000\).}
\label{fig:comparison}
\end{figure}
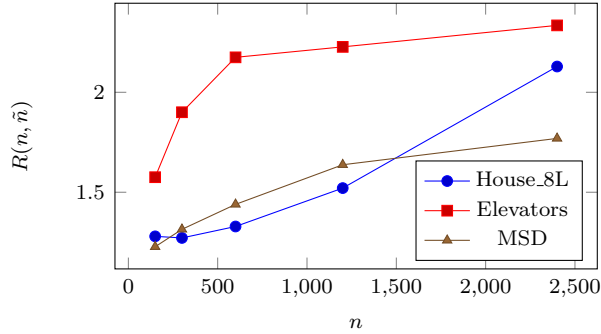
\section{Conclusion}\label{se:conclusion}
We have derived upper bounds on the expected excess risk of SGDIR. We first summarize our results in the noiseless case. When \(0<r<1/2\), our bounds for both averaged and non-averaged SGDIR are either strictly sharper than existing ones or hold under weaker assumptions. Moreover, whereas previous bounds saturate for \(r\geq 1/2\),  ours saturate only for \(r\geq 1\). In particular, when \(r=1/2\), we obtain bounds of order \(m^{-2}\log^{2}m  \), while, when \(r=1\),  we obtain bounds of order \(m^{-3+\epsilon}\) for any  \(\epsilon\ >\alpha^{-1}\). To the best of our knowledge, no existing algorithm has been shown to achieve such convergence rates under comparable assumptions. We also derive a lower bound of order \(m^{-2r-1}\) on the expected excess risk of any algorithm based on \(m\) training samples, which matches our upper bound, up to a \(\log^{2} m\) factor, when  \(-1/2\leq r\leq1/2\). Closing the gap between the upper  and lower  bounds for \(r>1/2\), using SGD-like algorithms or other algorithms such as ridge regression, is an open question. Noiseless SGD also arises in areas beyond machine learning, such as the averaging process on a graph \cite{bach2020tight}. Exploring applications of SGDIR to such areas is a question that deserves further research. 

We also provide an instance-based comparison of SGDIR and ridge regression in the noisy case under general assumptions. When the regularization parameter is of order \(1/n\) or larger, we show that the expected excess risk of SGDIR is no larger than that of ridge regression, up to a \(\log^{2}n\) factor, where \(n\) is the number of training samples. Whether a similar comparison holds for SGD is an open question. Numerical experiments on synthetic and real data are consistent with our theoretical findings. Extending our analysis beyond the squared loss is an interesting direction for future research.   
\appendix\section{Preliminary inequalities}\label{se:Preliminary}
This section gives preliminary results for use in subsequent proofs. Our analysis often uses the fact that if \(A\),  \(A'\), \(B\) and \(B'\) are self-adjoint operators and have a common orthonormal eigenbasis, with \(0\preccurlyeq A\preccurlyeq A'\) and  \(0\preccurlyeq B\preccurlyeq B'\), then \(AB\) and \(A'B'\) are self-adjoint with \(0\preccurlyeq AB\preccurlyeq A'B'\).  

The recursion  \eqref{eq:RecW_t} can be equivalently written as \begin{equation}\label{eq:RecW_tPt}
w_{t+1}-\theta^{*}=P_{t}(w _{t}-\theta^{*})-\gamma\lambda_{t}\theta^{*},
\end{equation} where \(P_{t}:=I-\gamma(  x_{t}\otimes x_{t}+\lambda_{t} I)\). For \(t\geq0\), we have \(E[P_{t}]= A_{t}\), where \(A_{t}:=I-\gamma(\Sigma+\lambda_{t})\). Lemma~\ref{le:expectationP_tNP_t} provides operator inequalities involving \(P_{t}\) and \(A_{t}\). \begin{lemma}\label{le:expectationP_tNP_t}For   \(t\geq0\), we have 
 \(A_{t}\succcurlyeq0\) and \begin{equation}\label{eq:EP_t^2Bound} E[P_{t}^{2}]\preccurlyeq (1-\gamma\lambda_{t})A_{t}\preccurlyeq(1-\gamma\lambda_{t})^{2}I.\end{equation} Moreover, for any self-adjoint operator \(M\) on \(\mathcal{H}\) satisfying
\( M \preccurlyeq \xi I\) and 
\( M \preccurlyeq \xi' \Sigma\), where  \(\xi\) and \(\xi'\) are nonnegative constants, we have, for  \(t \ge 0\),\begin{equation}\label{eq:expectationP_tNP_t}
E[P_{t}MP_{t}]\preccurlyeq (1-\gamma\lambda_{t})((1-\gamma\lambda_{t})\xi'  +\gamma\xi)\Sigma.
\end{equation} 
\end{lemma}
\begin{proof} 
 The second inequality in  \eqref{eq:EP_t^2Bound} holds because \(\gamma\lambda_{t}\leq1\) by \eqref{eq:GammaUpBound}.  For \(t\geq0\),\begin{eqnarray*}\nonumber E[P_{t}^{2}]&=&E[P_{t}]^{2}+E[(P_{t}-E[P_{t}])^{2}]\\\nonumber&=&A_{t}^{2}+\gamma^{2}(E[||x||^{2}  x\otimes x]-\Sigma^{2})\\&\preccurlyeq&A_{t}^{2}+\gamma A_{t}\Sigma
 \\&=&(1-\gamma\lambda_{t})A_{t},\end{eqnarray*}where the second equation follows from the equality\begin{equation}\label{eq:VarXxT}
E[(x\otimes x-\Sigma)^{2}]=E[||x||^{2}x\otimes x]-\Sigma^{2},
\end{equation}and the third from \eqref{eq:R2Assumption} and \eqref{eq:GammaUpBound}. This implies the first inequality in \eqref{eq:EP_t^2Bound}, which yields
 \(A_{t}\succcurlyeq0\). 

Consider now an operator \(M\) satisfying the conditions in the lemma. For \(t\geq0\),\begin{eqnarray*}\nonumber E[P_{t}MP_{t}]&=&E[P_{t}]ME[P_{t}]+E[(P_{t}-E[P_{t}])M(P_{t}-E[P_{t}])]\\
\nonumber&\preccurlyeq&\xi'  \Sigma A_{t}^{2}+\xi E[(P_{t}-E[P_{t}])^{2}]\\\nonumber&\preccurlyeq&\xi'  (1-\gamma\lambda_{t})^{2}\Sigma+\gamma^{2}\xi(E[||x||^{2}  x\otimes x]-\Sigma^{2})\\&\preccurlyeq&(1-\gamma\lambda_{t})((1-\gamma\lambda_{t})\xi'  +\gamma\xi)\Sigma,\end{eqnarray*}where the third equation follows from \(0\preccurlyeq A_{t}\preccurlyeq(1-\gamma\lambda_{t})I\), and the last from \eqref{eq:R2Assumption} and \eqref{eq:GammaUpBound}.  \end{proof}

Lemma~\ref{le:sumDistWtTheta*} provides a bound on the sum of \(||w_{t}-\theta^{*}||_{\Sigma}^{2}\) over intervals starting at \(j \ge m\).

\begin{lemma}\label{le:sumDistWtTheta*}For   \(k\geq0\) and  \(j\geq m\geq3\), \begin{equation}\label{eq:sumDistWtTheta*}
\gamma E\left[\sum^{j+k-1}_{t=j}||w_{t}-\theta^{*}||_{\Sigma}^{2}\right]\leq E[||w_{j}-\theta^{*}||^{2}_{I - A_{m}^{2k}}],
\end{equation}
where, by convention,  the left-hand side is zero when  \(k=0\).
\end{lemma}
\begin{proof}
We prove by induction on \(k\geq0\) that \eqref{eq:sumDistWtTheta*} holds for  all \(j\geq m\).  The base case  \(k=0\) is immediate. Assume now that  \eqref{eq:sumDistWtTheta*} holds for \(k\) and any \(j\geq m\). Replacing \(j\) by \(j+1\) shows that, for \(j\geq m\),  \begin{equation*}
\gamma E\left[\sum^{j+k}_{t=j+1}||w_{t}-\theta^{*}||_{\Sigma}^{2}\right]^{2}\leq E[||w_{j+1}-\theta^{*}||^{2}_{I - A_{m}^{2k}}].
\end{equation*}
Since \(w_{j+1}-\theta^{*}=P_{j}(w_{j}-\theta^{*})\) by \eqref{eq:RecW_tPt} and \(P_{j}\) is independent of \(w_{j}\), \begin{displaymath}
E[||w_{j+1}-\theta^{*}||^{2}_{I - A_{m}^{2k}}]=E[\langle w_{j}-\theta^{*},E[P_{j}(I - A_{m}^{2k})P_{j}]( w_{j}-\theta^{*})\rangle].
\end{displaymath}
 Moreover, as \(E[U]VE[U]\preccurlyeq E[UVU]\) for any deterministic self-adjoint  operator \(V\) with \(0\preccurlyeq V\) and any square-integrable self-adjoint operator \(U\),
 \begin{displaymath}
A_{m}^{2k+2}=E[P_{j}]A_{m}^{2k}E[P_{j}]\preccurlyeq E[P_{j} A_{m}^{2k}P_{j}].
\end{displaymath} Using \eqref{eq:EP_t^2Bound}, it follows that \(E[P_{j}(I - A_{m}^{2k})P_{j}]\preccurlyeq A_{m}-A_{m}^{2k+2}\). 
Consequently,  \begin{eqnarray*}
\gamma E\left[\sum^{j+k}_{t=j+1}||w_{t}-\theta^{*}||_{\Sigma}^{2}\right]^{2}&\leq& E\left[\langle w_{j}-\theta^{*},(A_{m}-A_{m}^{2k+2})( w_{j}-\theta^{*})\rangle\right]\\&=&E\left[||w_{j}-\theta^{*}||^{2}_{I - A_{m}^{2k+2}}\right]-\gamma E\left[||w_{j}-\theta^{*}||^{2}_{\Sigma}\right],
\end{eqnarray*}
where the second equation follows from   \(A_{m}=I-\gamma\Sigma\). Thus \eqref{eq:sumDistWtTheta*} holds for \(k+1\) and \(j\).     
\end{proof}
Lemma~\ref{le:PolyOpInequality} provides   general operator inequalities.

\begin{lemma}\label{le:PolyOpInequality} For $n \geq 1$ and any self-adjoint operator $M$ satisfying $0 \preccurlyeq M \preccurlyeq I$ and diagonalizable in an orthonormal basis, we have\begin{equation}\label{eq:PolyOpInequality}
 I-(I-M)^{n}\preccurlyeq2M\left(M+\frac{1}{n}\right)^{-1},
\end{equation}and\begin{equation}\label{eq:PolyOpInequalityBis}
I-(I-M)^{n}\preccurlyeq nM.
\end{equation} \end{lemma}
\begin{proof}As \(z\mapsto(1-z)^{n}\) is convex on the interval \([0,1]\), we have \(1-n z\leq(1-z)^{n} \) for \(  0\leq z\leq1\). Thus, when  \(n z \le 1\),   \begin{equation}\label{eq:ploynomialIneq}
1-(1-z)^{n}\leq\frac{2n z}{1+n z}.
\end{equation}   
When \(n z \ge 1\), \eqref{eq:ploynomialIneq} remains valid because the righthand side is lower bounded by \(1\). Consider now  an orthonormal eigenbasis \((u_{i})_{i\in J}\) of \(\mathcal{H}\) for the operator \(M\), and let \(a_{i}\) be the eigenvalue of \(M\) associated with \(u_{i}\), i.e., \(Mu_{i}=a_{i}u_{i}\) for \(i\in J\). Let \(B\) be the difference between the righthand side  and lefthand side of  \eqref{eq:PolyOpInequality}. For \(i\in J\),\begin{displaymath}
Bu_{i}=\left(\frac{2na_{i}}{1+na_{i}}-(1-(1-a_{i})^{n})\right)u_{i}.
\end{displaymath} By  \eqref{eq:ploynomialIneq},  the eigenvalues of \(B\) are nonnegative, which implies  \eqref{eq:PolyOpInequality}. A similar diagonalization argument implies \eqref{eq:PolyOpInequalityBis}.   
 \end{proof}
Lemma \ref{le:CombiningOperatorInequalities} shows how to combine two upper bounds on an operator into a single bound. 
\begin{lemma}\label{le:CombiningOperatorInequalities}
Let \(B\), \(C\) and \(D\) be self-adjoint operators on \(\mathcal{H}\) such that \(0\preccurlyeq B\preccurlyeq C\) and \(0\preccurlyeq B\preccurlyeq D\).   Assume further  that \(C+D\) is invertible that   \(C\) and \(D\) commute. Then \(CD(C+D)^{-1}\) is self-adjoint and \(B\preccurlyeq 2CD(C+D)^{-1}\).
\end{lemma}
\begin{proof}
 As \(C\) and \(D\) commute, \((C+D)^{-1}\) commutes with both \(C\) and \(D\), which implies that  \(CD(C+D)^{-1}\) is self-adjoint.  Assume first that \(B\) is invertible. Then \(C\) is invertible and \(C^{-1}\preccurlyeq B^{-1} \), since  \(0\preccurlyeq B\preccurlyeq C\). Similarly, \(D\) is invertible and \(D^{-1}\preccurlyeq B^{-1}\). Thus, \(C^{-1}+D^{-1}\preccurlyeq 2B^{-1}\). Inverting both sides yields \begin{displaymath}
B\preccurlyeq2(C^{-1}+D^{-1})^{-1}=2CD(C+D)^{-1}.
\end{displaymath}where the last equation follows from the fact that \(C\) and \(D\). In the general case, applying the preceding inequality to  \(B+\epsilon I\),   \(C+\epsilon I\) and \(D+\epsilon I\) for \(\epsilon>0\), and letting \(\epsilon\) go to \(0\), yields the desired bound.
\end{proof}
\section{Proof of Lemma \ref{le:TotalBiasAtT}}
We first analyze the bias of \(w_{t}\) 
 using Lemma~\ref{le:bias} and its variance using Lemma~\ref{le:SecondMomentEpsilon}. Both lemmas suppose that the assumptions  of Lemma \ref{le:TotalBiasAtT} hold.  \begin{lemma}\label{le:bias}
For \(t\geq 0\), we have \begin{equation}\label{eq:RecEta_tPt}
E[w_{t+1}]-\theta^{*}=A_{t}(E[w_{t}]-\theta^{*})-\gamma\lambda_{t}\theta^{*},
\end{equation} and \begin{equation}\label{eq:eta_tTheta*DiffSq}
||E[w_{t}]-\theta^{*}||^{2}\le2||\theta_{\Lambda}-\theta^{*}||^{2}+2(1-\gamma \Lambda )^{2\min (m,t)}||\theta_{\Lambda}||^{2}. 
\end{equation} For \(t\geq m\),
we have 
\begin{equation}\label{eq:Eta_tTheta*Diff}
||E[w_{t}]-\theta^{*}||_{\Sigma}\leq2\Lambda||A_{m}^{t-m}(\Sigma+\Lambda)^{-1}\theta^{*}||_{\Sigma}.
\end{equation}\end{lemma} 
\begin{proof}
 Taking expectations in  \eqref{eq:RecW_tPt}   and using the independence of \(P_{t}\) and \(w_{t}\) implies \eqref{eq:RecEta_tPt}. Because \(\theta_{\Lambda}=\Sigma(\Sigma+\Lambda)^{-1}\theta^{*}\), it follows by induction  that, for \(0\leq t\leq m\),  \begin{equation}\label{eq:E[w_t]Value}
E[w_{t}]=\theta_{\Lambda}-A_{0}^{t}\theta_{\Lambda}. 
\end{equation}As \(0\preccurlyeq A_{0}\preccurlyeq(1-\gamma \Lambda )I\) by Lemma~\ref{le:expectationP_tNP_t},
    this yields \(||E[w_{t}]-\theta_{\Lambda}||\leq(1-\gamma \Lambda )^{t}||\theta_{\Lambda}||\) for \(0\leq t\leq m\).  Using the inequality \(||u+v||^{2}\leq2||u||^{2}+2||v||^{2}\)
 for \(u,v\in\mathcal{H}\), this shows that  \eqref{eq:eta_tTheta*DiffSq} holds  for \(0\leq t\leq m\). Moreover, as \(0\preccurlyeq A_m\preccurlyeq I\) by Lemma~\ref{le:expectationP_tNP_t}, it follows from \eqref{eq:RecEta_tPt} and an induction argument that \(||E[w_{t}]-\theta^{*}||\leq||E[w_{m}]-\theta^{*}||\)  for \(t\geq m\).
Hence \eqref{eq:eta_tTheta*DiffSq} holds for \(t\geq m\) as well.

Replacing \(\theta_{\Lambda}\) by \(\Sigma(\Sigma+\Lambda)^{-1}\theta^{*}\) in \eqref{eq:E[w_t]Value} shows that\begin{displaymath}
E[w_{m}]=\theta^{*}-\Lambda(I+\Lambda^{-1}A_{0}^{m}\Sigma )(\Sigma+\Lambda)^{-1}\theta^{*}.
\end{displaymath}Using \eqref{eq:RecEta_tPt} and an induction argument, it follows that, for \(t\geq m\),
\begin{equation}\label{eq:Ew_tMinusTheta*}
E[w_{t}]=\theta^{*}-\Lambda(I+\Lambda^{-1}A_{0}^{m}\Sigma )A_{m}^{t-m}(\Sigma+\Lambda)^{-1}\theta^{*}.
\end{equation} As  \(0\preccurlyeq A_{0}\preccurlyeq(1-\gamma \Lambda )I\) by Lemma~\ref{le:expectationP_tNP_t} and \(A_{0}\) and \(\Sigma\) have a common orthonormal eigenbasis,   we have \begin{equation*}
0\preccurlyeq\Lambda^{-1}A_{0}^{m}\Sigma\preccurlyeq\Lambda^{-1}R^{2}(I-\gamma\Lambda )^{m} I\preccurlyeq \gamma R^{2}I,
\end{equation*} where the second inequality follows from \eqref{eq:SigmaLeR^2}, and the third from  \(1+z\leq e^{z}\) for \(z\in\mathbb{R}\) together with \(\Lambda\geq \log(m)/(\gamma m)\). Using  \eqref{eq:GammaUpBound}, it follows that \(0\preccurlyeq\Lambda^{-1}A_{0}^{m}\Sigma\preccurlyeq I\) which, combined with \eqref{eq:Ew_tMinusTheta*}, yields \eqref{eq:Eta_tTheta*Diff}.
\end{proof}

Set \(\epsilon_{t}=w_{t}-E[w_{t}]\) for \(t\geq0\). Combining \eqref{eq:RecW_tPt} and   \eqref{eq:RecEta_tPt} shows that, for \(t\geq0\),\begin{equation}\label{eq:recEps}
\epsilon_{t+1}=P_{t}\epsilon_{t}+(P_{t}-A_{t})(E[w_{t}]-\theta^{*}).
\end{equation} 
 
\begin{lemma}\label{le:SecondMomentEpsilon}For \(t\geq m\), we have \(E[||\epsilon _{t}||^{2}]\le4\gamma ^{2}R^{2}\Lambda^{2}||(\Sigma+\Lambda)^{-1}\theta^{*}||^{2}_{\Sigma}t\). 
\end{lemma}
\begin{proof}
Because \(E[\epsilon_{t}]=0\) and \(\epsilon_{t}\) is independent of \(P_{t}\)  for \(t\geq0\),  \begin{eqnarray*}E[||\epsilon_{t+1}||^{2}]&=& E[||P_{t}\epsilon_{t}||^{2}+E[||(P_{t}-A_{t})(E[w_{t}]-\theta^{*})||^{2}]\\&
=&E[\langle P_{t}^{2}\epsilon_{t},\epsilon_{t}\rangle]+\gamma^{2}E[||(x\otimes x-\Sigma)(E[w_{t}]-\theta^{*})||^{2}]\\&\leq&(1-\gamma \lambda_{t})^{2}E[||\epsilon _{t}||^{2}]+\gamma^{2}R^{2}||E[w_{t}]-\theta^{*}||^{2}_{\Sigma},
\end{eqnarray*}where the first equation follows from    \eqref{eq:recEps}  and the equality \(E[||U+V||^{2}]=E[||U||^{2}]+E[||V||^{2}]\) for square-integrable random vectors \(U\) and \(V\) with \(E[\langle U,V\rangle]=0\), the second from the definition of \(P_{t}\) and \(A_{t}\), and the last from \eqref{eq:EP_t^2Bound}, \eqref{eq:R2Assumption}, and \eqref{eq:VarXxT}.  Using \eqref{eq:eta_tTheta*DiffSq}, it follows that, for \(t\geq0\),  \begin{equation}
E[||\epsilon_{t+1}||^{2}]\leq(1-\gamma \lambda_{t})^{2}E[||\epsilon _{t}||^{2}]+\eta_{t},
\end{equation}where\begin{equation*}
\eta_{t}=2\gamma ^{2}R^{2}(||\theta_{\Lambda}-\theta^{*}||_{\Sigma}^{2}+(1-\gamma\Lambda)^{2\min (m-1,t)}||\theta_{\Lambda}||_{\Sigma}^{2}).
\end{equation*} 
Because \((1-\gamma\lambda_{t+1})^{2}\eta_{t}\leq\eta_{t+1}\) for \(t\geq0\), it follows by induction that \((1-\gamma\lambda_{t})^{2}E[||\epsilon _{t}||^{2}]\leq \eta_{t}t\) for \(t\geq0\). Moreover, \begin{equation*}
||\theta_{\Lambda}||_{\Sigma}=||\Sigma(\Sigma+\Lambda)^{-1}\theta^{*}||_{\Sigma}\leq R^{2}||(\Sigma+\Lambda)^{-1}\theta^{*}||_{\Sigma}\leq \Lambda(1-\gamma\Lambda)m||(\Sigma+\Lambda)^{-1}\theta^{*}||_{\Sigma},
\end{equation*}
where the second equation follows from \eqref{eq:SigmaLeR^2}, and the last from    \eqref{eq:GammaUpBound} and the inequality \(\Lambda\geq \log(m)/(\gamma m)\).  Furthermore, because \(1+z\leq e^{z}\) for \(z\in\mathbb{R}\), we have \((1-\gamma\Lambda)^{m}\leq m^{-1}\). As \(\theta^{*}-\theta_{\Lambda}=\Lambda(\Sigma+\Lambda)^{-1}\theta^{*}\), we conclude that\begin{equation}\label{eq:EtamBound}
\eta_{m}\leq4\gamma ^{2}R^{2}\Lambda^{2}||(\Sigma+\Lambda)^{-1}\theta^{*}||^{2}_{\Sigma}.
\end{equation}
This completes the proof. 
\end{proof}
We now prove Lemma \ref{le:TotalBiasAtT}.  
 By  Lemma \ref{le:sumDistWtTheta*}, \begin{eqnarray}\label{eq:TotalBiasAtTEndProof}\nonumber\gamma mE\left[||w_{N}-\theta^{*}||_{\Sigma}^{2}\right]&\leq& E\left[||w_{2m}-\theta^{*}||_{I-A_{m}^{2m}}^{2}\right]\\&=\nonumber&||E[w_{2m}]-\theta^{*}||_{I-A_{m}^{2m}}^{2}+E\left[||\epsilon_{2m}||^{2}_{I-A_{m}^{2m}}\right]\\
&\le&8\gamma \Lambda^{2}m||A_{m}^{m}(\Sigma+\Lambda)^{-1}\theta^{*}||_{\Sigma}^{2}+E\left[||\epsilon_{2m}||^{2}_{I-A_{m}^{2m}}\right].\end{eqnarray} The second equation is a standard bias--variance decomposition, while the third follows from \eqref{eq:Eta_tTheta*Diff} together with the inequality \(I-A_{m}^{2m}\preccurlyeq2\gamma m\Sigma\), which is a special case of  \eqref{eq:PolyOpInequalityBis}.
We have \(||A_{m}^{m}(\Sigma+\Lambda)^{-1}\theta^{*}||_{\Sigma}\leq||(\Sigma+\Lambda)^{-1}\theta^{*}||_{\Sigma}\) since    \(0\preccurlyeq A_{m}\preccurlyeq I\). Moreover,  Lemma~\ref{le:SecondMomentEpsilon} together with \eqref{eq:GammaUpBound} yields  \begin{displaymath}
E[||\epsilon _{2m}||^{2}]\le8\gamma\Lambda^{2}m||(\Sigma+\Lambda)^{-1}\theta^{*}||^{2}_{\Sigma}.
\end{displaymath} Combining the preceding inequalities implies \eqref{eq:TotalBiasAtT}. 
\section{Proof of Lemma \ref{le:TotalBiasAtTBis}}
We first prove the following bound on the covariance of \(\epsilon_{t}\).
\begin{lemma}\label{le:EpsCovariance}For \(t\geq m\), 
\begin{equation*}
E[\epsilon_{t}\otimes\epsilon_{t}]\preccurlyeq  8 
\gamma R^{2}\Lambda^{2}||(\Sigma+\Lambda)^{-1}\theta^{*}||^{2}\Sigma\left(\Sigma+\frac{1}{2\gamma t}\right)^{-1}.
\end{equation*}
\end{lemma}
\begin{proof}
  For \(t\geq0\), define  \(C_{t}=E[\epsilon_{t}\otimes\epsilon_{t}]\), and let \(U_{t}=P_{t}\epsilon_{t}\) and \(V_{t}=(P_{t}-A_{t})(E[w_{t}]-\theta^{*})\) be the components of the righthand side of  \eqref{eq:recEps}. Because \(E[\epsilon_{t}]=0\) and    \(\epsilon_{t}\) and \(P_{t}\) are independent for \(t\geq0\), we have \(E[U_{t} \otimes V_{t}] = 0\). Consequently,  
\(E[(U_{t} + V_{t}) \otimes (U_{t} + V_{t})] = E[U_{t} \otimes U_{t} +V_{t} \otimes V_{t}]\).  
Thus, for  \(t\geq 0\), 
\begin{eqnarray*}\nonumber C_{t+1}&=&E[(P_{t}\epsilon_{t})\otimes(P_{t}\epsilon_{t})]+E[V_{t} \otimes V_{t}]\\
&\preccurlyeq&\nonumber E[P_{t}(\epsilon_{t}\otimes\epsilon_{t})P_{t}]+||E[w_{t}]-\theta^{*}||^{2}E[(P_{t}-A_{t})^{2}]\\&=&\nonumber E[P_{t}C_{t}P_{t}]+\gamma ^{2}||E[w_{t}]-\theta^{*}||^{2}E[(x\otimes x-\Sigma)^{2}]\\&\preccurlyeq&E[P_{t}C_{t}P_{t}]+\gamma ^{2}||E[w_{t}]-\theta^{*}||^{2}R^{2}\Sigma,\end{eqnarray*}where the second equation follows from the relations \((Mu)\otimes(Mu)=M(u\otimes u)M\) for \(u\in\mathcal{H}\) and self-adjoint operator \(M\), and  \((E[w_{t}]-\theta^{*})\otimes(E[w_{t}]-\theta^{*})\preccurlyeq||E[w_{t}]-\theta^{*}||^{2}I\),  the third  from the independence of  \(\epsilon_{t}\) and \(P_{t}\), and the last from \eqref{eq:R2Assumption} and \eqref{eq:VarXxT}. Using \eqref{eq:eta_tTheta*DiffSq},  it follows that \begin{equation}\label{eq:CovEpsilonRec}
C_{t+1}\preccurlyeq E[P_{t}C_{t}P_{t}]+\gamma\xi_{t}\Sigma,
\end{equation}
where\begin{equation*}
\xi_{t}=2\gamma R^{2}(||\theta_{\Lambda}-\theta^{*}||^{2}+(1-\gamma\Lambda)^{2\min (m-1,t)}||\theta_{\Lambda}||^{2}).
\end{equation*}We now prove by induction that \(C_{t}\preccurlyeq \xi_{t}(1-\gamma \lambda _{t})^{-1}I\) and \(C_{t}\preccurlyeq 2\gamma\xi_{t} t(1-\gamma \lambda _{t})^{-2}\Sigma \) for \(t\geq0\).  The base case \(t=0\) clearly holds. Assume the induction hypothesis holds for \(t\). Then\begin{eqnarray*}C_{t+1}&\preccurlyeq&\frac{\xi_{t}}{1-\gamma \lambda _{t}}E[P_{t}^{2}]+\gamma\xi_{t}\Sigma\\
&\preccurlyeq& \xi_{t}A_{t}+\gamma\xi_{t}\Sigma\\&=&(1-\gamma \lambda _{t})\xi_{t}I\preccurlyeq\frac{\xi_{t+1}}{1-\gamma \lambda _{t+1}}I,
\end{eqnarray*}
where the first equation follows from \eqref{eq:CovEpsilonRec} and the induction hypothesis, the second from \eqref{eq:EP_t^2Bound}, and the last from the inequality \((1-\gamma\lambda_{t})(1-\gamma\lambda_{t+1})\xi_{t}\leq\xi_{t+1}\) for \(t\geq0\).
Similarly, it follows from the induction hypothesis, \eqref{eq:CovEpsilonRec}, and \eqref{eq:expectationP_tNP_t} applied to \(C_t\) with \(\xi=\xi_t(1-\gamma\lambda_t)^{-1}\) and \(\xi'=2\gamma\xi_t t(1-\gamma\lambda_t)^{-2}\), that \begin{equation*}
C_{t+1}\preccurlyeq2\gamma \xi_{t}(t+1)\Sigma\preccurlyeq \frac{2\gamma\xi_{t+1} (t+1)}{(1-\gamma \lambda _{t+1})^{2}}\Sigma ,
\end{equation*}  where the second equation follows from the inequality \((1-\gamma\lambda_{t+1})^{2}\xi_{t}\leq\xi_{t+1}\) for \(t\geq0\). Thus, the induction hypothesis holds for \(t+1\). 
By Lemma \ref{le:CombiningOperatorInequalities}, it follows that  \(C_{t}\preccurlyeq 2\xi_{m}\Sigma(\Sigma+(2\gamma t)^{-1})^{-1}\) for \(t\geq m.\) Moreover, a calculation similar to that leading to \eqref{eq:EtamBound}  shows that \(\xi_{m}\leq4\gamma R^{2}\Lambda^{2}||(\Sigma+\Lambda)^{-1}\theta^{*}||^{2}\), thereby completing the proof.
\end{proof}
Lemma \ref{le:EpsilonSecondMomentBis} provides a uniform bound on \(E\left[||\epsilon_{k}||^{2}_{I-A_{m}^{2m}}\right]\) over a range of values of \(k\).
\begin{lemma}\label{le:EpsilonSecondMomentBis}For  \(m\leq k\leq 3m\), we have\begin{displaymath}
E\left[||\epsilon_{k}||^{2}_{I-A_{m}^{2m}}\right]\le16\gamma^{2} R^{2}\Lambda^{2}\tilde\Lambda||(\Sigma+\Lambda)^{-1}\theta^{*}||^{2}m.
\end{displaymath} 
\end{lemma}
\begin{proof}
Using the identity  \(\tr[u\otimes v]=\langle u,v\rangle\) for \(u,v\in\mathcal{H}\) and Lemma \ref{le:EpsCovariance}, we have\begin{eqnarray*}\nonumber
E\left[||\epsilon_{k}||^{2}_{I-A_{m}^{2m}}\right]&=&E\left[\tr\left[(I-A_{m}^{2m})\epsilon_{k}\otimes\epsilon_{k}\right]\right]
\\\nonumber&\le& 8\gamma R^{2}\Lambda^{2}||(\Sigma+\Lambda)^{-1}\theta^{*}||^{2} \tr\left[(I-A_{m}^{2m})\Sigma\left(\Sigma+\frac{1}{6\gamma m}\right)^{-1}\right]\\
&\le& 16\gamma R^{2}\Lambda^{2}||(\Sigma+\Lambda)^{-1}\theta^{*}||^{2} \tr\left[\Sigma^{2} \left( \Sigma+\frac{1}{2\gamma m}\right)^{-1}\left(\Sigma+\frac{1}{6\gamma m}\right)^{-1}\right],
\end{eqnarray*}
where the last inequality holds because, by \eqref{eq:PolyOpInequality}, \begin{equation*}
 I-A_{m}^{2m}\preccurlyeq2 \Sigma \left( \Sigma+\frac{1}{2\gamma m}\right)^{-1}.
\end{equation*}
As \((z+1/2)(z+1/6)\geq(z+1/4)^{2}\) for \(z\geq0\), we conclude by diagonalization that \begin{displaymath}
E\left[||\epsilon_{k}||^{2}_{I-A_{m}^{2m}}\right]\le16\gamma R^{2}\Lambda^{2}||(\Sigma+\Lambda)^{-1}\theta^{*}||^{2} \tr\left[\Sigma^{2}\left(\Sigma+\frac{1}{4\gamma m}\right)^{-2}\right].
\end{displaymath}
\end{proof}
Combining \eqref{eq:TotalBiasAtTEndProof} with the inequality \(||A_{m}^{m}(\Sigma+\Lambda)^{-1}\theta^{*}||_{\Sigma}\le\Lambda^{-1}||A_{m}^{m}\theta^{*}||_{\Sigma}\) and  Lemma \ref{le:EpsilonSecondMomentBis} yields \eqref{eq:TotalBiasAtTBis}. 
\section{Proof of Theorem \ref{th:noiseless}}\label{se:noiseless}

Lemmas \ref{le:SecondMomentEpsilon} and  \ref{le:EpsilonSecondMomentBis}  provide two alternative bound on \(E\left[||\epsilon_{2m}||^{2}_{I-A_{m}^{2m}}\right]\). We   combine these bounds using Lemma \ref{le:CombiningOperatorInequalities}.   Using \eqref{eq:RecW_t}, it can be shown by induction that, for \(t\geq0\), there is a random operator   \(M_{t}\)  on \(\mathcal{H}\) that does not depend on \(\theta^{*}\) such that the operator \(E[M_{t}\otimes M_{t}]\) is  well-defined and  \(w_{t}= M_{t}\theta^{*}\). Fix now \(t\in [2m,3m-1]\). As \(\epsilon_{t}=(M_{t}-E[M_{t}])\theta^{*}\), we have \(E\left[||\epsilon_{t}||^{2}_{I-A_{m}^{2m}}\right]=\langle\theta^*,B\theta^*\rangle\),  where \(B\)  is a positive semidefinite self-adjoint operator that does not depend on \(\theta^{*}\). It follows from Lemma \ref{le:SecondMomentEpsilon} that \(\langle\theta^*,B\theta^*\rangle\leq\langle\theta^*,C\theta^*\rangle\),  where \(C=16\gamma ^{2}R^{2}\Lambda^{2} m(\Sigma+\Lambda)^{-2}\Sigma\).   Since this inequality holds for any \(\theta^{*}\in\mathcal{H}\), we conclude that \(B\preccurlyeq C\). Similarly,   Lemma~\ref{le:EpsilonSecondMomentBis} implies that \(B\preccurlyeq D\), where \(D=16\gamma ^{2}R^{2}\Lambda^{2}\tilde\Lambda m(\Sigma+\Lambda)^{-2}\). By Lemma \ref{le:CombiningOperatorInequalities}, it follows that\begin{eqnarray}\label{eq:EpsilonSecondMomentComb}\nonumber E\left[||\epsilon_{t}||^{2}_{I-A_{m}^{2m}}\right]&\leq& 2\langle\theta^*, CD(C+D)^{-1}\theta^*\rangle\\&=&32\gamma ^{2}R^{2}\Lambda^{2}\tilde\Lambda m||(\Sigma+\Lambda)^{-1}(\Sigma+\tilde\Lambda)^{-1/2}\theta^*||_{\Sigma}^{2}.\end{eqnarray}
Together with   \eqref{eq:TotalBiasAtTEndProof} and the inequality \(||A_{m}^{m}(\Sigma+\Lambda)^{-1}\theta^{*}||_{\Sigma}\le\Lambda^{-1}||A_{m}^{m}\theta^{*}||_{\Sigma}\), this concludes the proof. 
\section{Proof of Lemma \ref{le:BiasVarDecomp}}
The proof of Lemma~\ref{le:BiasVarDecomp} builds on ideas developed in \cite{kakade2023benign}, who provide tight bounds on SGD in a context without regularization. In particular, Lemma \ref{le:Delta_tDelta_tTBound} provides  a bound on \(E[\delta_{t}\otimes \delta_{t}]\)    similar in spirit to   \cite[Lemma B.5]{kakade2023benign}, and some calculations in the proof of Lemma \ref{le:sumCov}
 resemble those  in \cite[Lemma B.6]{kakade2023benign}. However, the proof of Lemma \ref{le:Delta_tDelta_tTBound}  is based on Lemmas \ref{le:expectationP_tNP_t}  and \ref{le:CombiningOperatorInequalities}, whereas the proof of their Lemma B.5 relies on operators acting on symmetric matrices.  
    
\begin{lemma}\label{le:Delta_tDelta_tTBound}
For \(0\leq t\leq3m\), we have \begin{displaymath}
E[\delta_{t}\otimes \delta_{t}]\preccurlyeq2\gamma\sigma^{2}\Sigma\left(\Sigma+\frac{1}{6\gamma m}\right)^{-1}.
\end{displaymath}
\end{lemma}
\begin{proof}
We show by induction on \(t\geq0\) that   \(C_{t}\preccurlyeq \gamma\sigma^{2}I\) and \(C_{t}\preccurlyeq2\gamma^{2}\sigma^{2}t\Sigma\), where \(C_{t}=E[\delta_{t}\otimes \delta_{t}]\).  The base case \(t=0\)   is trivial. Assume the induction hypothesis holds for  \(t\).  
 Equation~\eqref{eq:RecDelta} can be rewritten as\begin{equation}\label{eq:deltaRecWithPt}
\delta_{t+1} = P_{t}\delta_{t} + v_{t}.
\end{equation} As \(\delta_{t}\) is independent of  \((P_{t}, v_{t})\) and \(E[\delta_{t}]=0\), we have \(E[(P_{t}\delta_{t})\otimes v_{t}]=0\). 
Consequently, \begin{eqnarray*}
C_{t+1}&=& E[(P_{t}\delta_{t})\otimes(P_{t}\delta_{t})]+E[v_{t}\otimes v_{t}]\\&\preccurlyeq& E[P_{t}(\delta_{t}\otimes\delta_{t})P_{t}]+\gamma^{2} \sigma^{2}\Sigma\\
&\preccurlyeq&E[P_{t}C_{t}P_{t}]+\gamma^{2} \sigma^{2}\Sigma,
\end{eqnarray*}
where the second equation follows from  \eqref{eq:SigmaAssumption}, and the third   the independence of \(P_{t}\) and \(\delta_{t}\). As \(C_{t}\preccurlyeq \gamma\sigma^{2}I\) by the induction hypothesis, it follows that  \begin{eqnarray*}
C_{t+1}&\preccurlyeq& \gamma\sigma^{2}E[P_{t}^{2}]+\gamma^{2} \sigma^{2}\Sigma\\&\preccurlyeq& \gamma\sigma^{2}(I - \gamma\Sigma)+\gamma^{2} \sigma^{2}\Sigma=\gamma\sigma^{2}I, \end{eqnarray*}
where the second equation follows from \eqref{eq:EP_t^2Bound}.
Similarly, using the  induction hypothesis and applying \eqref{eq:expectationP_tNP_t} to \(C_{t}\) with \(\xi=\gamma\sigma^{2}\) and  \(\xi'=2\gamma^{2}\sigma^{2}t\), we obtain   \begin{equation*}
C_{t+1}\preccurlyeq (2\gamma^{2}\sigma^{2}t+\gamma^{2} \sigma^{2})\Sigma+\gamma^{2} \sigma^{2}\Sigma=2\gamma^{2}\sigma^{2}(t+1)\Sigma.\end{equation*}  Thus, the induction hypothesis  holds for \(t+1\). We conclude that    \(C_{t}\preccurlyeq \gamma\sigma^{2}I\)  and \(C_{t}\preccurlyeq6\gamma^{2}\sigma^{2}m\Sigma\)    for \(0\leq t\leq3m\). Using Lemma \ref{le:CombiningOperatorInequalities} concludes the proof. 
\end{proof}

\begin{lemma}\label{le:sumCov}We have \begin{displaymath}
E\left[\left\|\sum_{j = 2m}^{3m-1 } \delta_{j}\right\|^{2}_{\Sigma}\right]\leq 2\gamma^{-1}\sum^{3m-1}_{j=2m}E[||\delta_{j}||^{2}_{I-A_{m}^{m}}].
\end{displaymath} \end{lemma}\begin{proof}We have\begin{eqnarray*}
E\left[\left\|\sum_{j = 2m}^{3m-1 } \delta_{j}\right\|^{2}_{\Sigma}\right]&=&E\left[\sum^{3m-1}_{j=2m}||\delta_{j}||^{2}_{\Sigma}+2\sum^{3m-1}_{j=2m}\sum^{3m-1}_{t=j+1}\langle\delta_{j},\Sigma\delta_{t}\rangle\right]\\
&\le&2\sum^{3m-1}_{j=2m}\sum^{3m-1}_{t=j}E[\langle\delta_{j},\Sigma\delta_{t}\rangle].\end{eqnarray*}On the other hand, by \eqref{eq:deltaRecWithPt},
for  \(2m\leq j\leq t\),     \begin{equation*}
E[\delta_{t+1}|\delta_{j}]=E[P_{t}\delta_{t}|\delta_{j}]+E[v_{t}|\delta_{j}]=E[P_{t}|\delta_{j}]E[\delta_{t}|\delta_{j}],
\end{equation*}
where the second equality follows from the facts  \(P_{t}\) and \(\delta_{t}\) are independent, conditional on \(\delta_{j}\),   and  \(v_{t}\) is independent of \(\delta_{j}\) and has zero mean. Because   \(P_{t}\) and \(\delta_{j}\) are independent, we have \(E[P_{t}|\delta_{j}]=E[P_{t}]=A_{m}\).   Thus, \(E[\delta_{t+1}|\delta_{j}]=A_{m}E[\delta_{t}|\delta_{j}]\). 
It follows by induction  that \(E[\delta_{t}|\delta_{j}]=A_{m}^{t-j}\delta_{j}\) for \(2m\leq j\leq t\),
which implies that \begin{displaymath}
E[\langle\delta_{j},\Sigma\delta_{t}\rangle|\delta_{j}]=\langle\delta_{j},\Sigma E[\delta_{t}|\delta_{j}]\rangle=\langle\delta_{j},\Sigma A_{m}^{t-j}\delta_{j}\rangle.
\end{displaymath}Taking expectations yields\begin{equation*}
E[\langle\delta_{j},\Sigma\delta_{t}\rangle]=E[\langle\delta_{j},\Sigma A_{m}^{t-j}\delta_{j}\rangle].
\end{equation*}Summing over \(t\) implies that, for   \(j\geq2m\), \begin{equation*}\sum^{3m-1}_{t=j}E[\langle\delta_{j},\Sigma\delta_{t}\rangle]=\gamma^{-1} E[\langle\delta_{j},(I-A_{m}^{3m-j})\delta_{j}\rangle]\leq\gamma^{-1}E[||\delta_{j}||^{2}_{I-A_{m}^{m}}],
\end{equation*}
where the first equation follows from the identity \(\sum^{3m-1}_{t=j}\Sigma A_{m}^{t-j}=\gamma^{-1} (I-A_{m}^{3m-j}\)), and the second from the inequality \(A_{m}^{m}\preccurlyeq A_{m}^{3m-j}\), which is a consequence of   \eqref{eq:EP_t^2Bound}.   
\end{proof}

 We now prove Lemma \ref{le:BiasVarDecomp}. Let \(j\in[2m,3m-1]\). Because \(\tr[u\otimes v]=\langle u,v\rangle\) for \(u,v\in\mathcal{H}\),  \begin{displaymath}
E[||\delta_{j}||^{2}_{I-A_{m}^{m}}]\leq E\left[\tr\left[(I-A_{m}^{m})\delta_{j}\otimes\delta_{j}\right]\right].
\end{displaymath}On the other hand, by \eqref{eq:PolyOpInequality}, \begin{equation*}
 I-A_{m}^{m}\preccurlyeq2 \Sigma \left( \Sigma+\frac{1}{\gamma m}\right)^{-1}.
\end{equation*}Together with Lemma \ref{le:Delta_tDelta_tTBound}, we conclude that \begin{displaymath}
E[||\delta_{j}||^{2}_{I-A_{m}^{m}}]\leq 4\gamma\sigma^{2}\tr\left[ \Sigma^{2} \left( \Sigma+\frac{1}{\gamma m}\right)^{-1}\left(\Sigma+\frac{1}{6\gamma m}\right)^{-1}\right]\leq4\gamma\sigma^{2} \tr\left[ \Sigma^{2} \left( \Sigma+\frac{1}{4\gamma m}\right)^{-2}\right],
\end{displaymath} where the second equation follows from the inequality \((z+1)(z+1/6)\geq(z+1/4)^{2}\) for \(z\geq 0\). Using Lemma~\ref{le:sumCov} completes the proof.
\section{Proof of Theorem \ref{th:Main}}Using \eqref{eq:Eta_tTheta*Diff} and  \((\Sigma+\Lambda)^{-1}\preccurlyeq\Lambda^{-1}I\) implies that \(||E[w_{t}]-\theta^{*}||_{\Sigma}\leq2||A_{m}^{t-m}\theta^{*}||_{\Sigma}\) for \(t\geq m\). As \(0\preccurlyeq A_{m}\preccurlyeq I\), it follows that  \(||E[w_{t}]-\theta^{*}||_{\Sigma}\leq2||A_{m}^{m}\theta^{*}||_{\Sigma}\) for \(t\geq 2m\). By the convexity of the function \(||.||^{2}_{\Sigma}\), this implies \eqref{eq:BiasOfAvgBound}. Moreover, a calculation similar to that in the proof of Lemma~\ref{le:sumCov}, using \eqref{eq:recEps} in place of  \eqref{eq:deltaRecWithPt}, shows that \begin{displaymath}
E\left[\left\|\sum_{j = 2m}^{3m-1 }\epsilon_{j}\right\|^{2}_{\Sigma}\right]\leq 2\gamma^{-1}\sum^{3m-1}_{j=2m}E[||\epsilon_{j}||^{2}_{I-A_{m}^{m}}].
\end{displaymath}Together with \eqref{eq:EpsilonSecondMomentComb} and the inequality \(A_{m}^{2m}\preccurlyeq A_{m}^{m}\), this yields\begin{displaymath}
E\left[\left\|\bar\epsilon_{2m:3m}\right\|^{2}_{\Sigma}\right]\leq 64\gamma R^{2}\Lambda^{2}\tilde\Lambda ||(\Sigma+\Lambda)^{-1}(\Sigma+\tilde\Lambda)^{-1/2}\theta^*||_{\Sigma}^{2}.
\end{displaymath}
Equation \eqref{eq:VarOfAvgBound} now follows by combining this bound with Lemma~\ref{le:BiasVarDecomp} and the relation 
\begin{displaymath}\var _{\Sigma}(\bar\theta_{2m:3m})=E\left[\left\|\bar\epsilon_{2m:3m}+\bar\delta_{2m:3m}\right\|^{2}_{\Sigma }\right]\leq 2E\left[\left\|\bar\epsilon_{2m:3m}\right\|^{2}_{\Sigma }\right]+2E\left[\left\|\bar\delta_{2m:3m}\right\|^{2}_{\Sigma }\right].
\end{displaymath}
\section{Proof of Lemma \ref{le:OperatorCvxSq}}
We first prove the following.  
\begin{lemma}\label{le:detOpIneq}Let \(A\) and \(B\) be self-adjoint operators on \(\mathcal{H}\) such that  $0\preccurlyeq B$ and   $0\preccurlyeq I-B\preccurlyeq A\preccurlyeq I$.  Then \(ABA\preccurlyeq4 B\). 
\end{lemma}
\begin{proof}Since $0\preccurlyeq B$, we have \(0\preccurlyeq(2I - A) B (2I - A)\). As \begin{displaymath}
2 B+2(I-A)B(I-A)-ABA=(2I - A) B (2I - A),
\end{displaymath}this yields\begin{equation*}
ABA\preccurlyeq2 B+2(I-A)B(I-A).
\end{equation*}
Moreover,\begin{displaymath}
(I-A)B(I-A)\preccurlyeq(I-A)^{2}\preccurlyeq I-A\preccurlyeq B,
\end{displaymath} where the first inequality follows from  \(B\preccurlyeq I\) and the second from \(0\preccurlyeq A\preccurlyeq I\). This completes the proof.
\end{proof}
We now prove Lemma \ref{le:OperatorCvxSq}. 
Set \(P= E[(U+I)^{-1}]\).
Since the map $X \mapsto (X+I)^{-1}$ is operator convex on self-adjoint positive semidefinite operators (see, e.g. \cite[Lemma 8]{mourtada2022elementary}), by Jensen's operator inequality,
\begin{displaymath}
(V+I)^{-1} = (E[U+I])^{-1} \preccurlyeq E[(U+I)^{-1}] = P.
\end{displaymath}
Since $0\preccurlyeq U $, we have $0 \preccurlyeq (U+I)^{-1} \preccurlyeq I$ and therefore
\(
(V+I)^{-1} \preccurlyeq P \preccurlyeq I
\).
This implies that   $0\preccurlyeq B$  and   $0\preccurlyeq I-B\preccurlyeq A\preccurlyeq I$, where  $B = I - (V + I)^{-1}$ and \(A = (V + I)^{-1/2} P^{-1} (V + I)^{-1/2}\).  By Lemma \ref{le:detOpIneq}, it follows that  \(B\preccurlyeq4A^{-1} BA^{-1}\). After some simplifications, we conclude that \(
V (V+I)^{-2} \preccurlyeq4PVP\). Moreover, because \(0\preccurlyeq V\), \begin{equation*}
0\preccurlyeq ((U+I)^{-1}-P)V((U+I)^{-1}-P) .
\end{equation*}Taking expectations implies that\begin{equation*}
P V P\preccurlyeq E[(U+I)^{-1} V (U+I)^{-1}] ,
\end{equation*}
which completes the proof.
\bibliography{poly}
\end{document}

%% file: alpha=1,25s=1.tex
\begin{tikzpicture}
\begin{axis}[
    width=\linewidth,
    height=0.85\linewidth,
    xlabel={$n$},
    xmode=log,
    ylabel={Test error},
    ymode=log,
]
\addplot[mark=*,color=blue] coordinates {
    (150,0.00295) (300,0.00118) (600,0.000495) (1200,0.000204)
    (2400,8.73e-05) (4800,3.71e-05) (9600,1.55e-05)
};

\addplot[mark=square*,color=red] coordinates {
    (150,0.0027) (300,0.000861) (600,0.000272) (1200,8.27e-05)
    (2400,2.5e-05) (4800,7.43e-06) (9600,2.22e-06)
};

\addplot[mark=triangle*,color=brown] coordinates {
    (150,0.00375) (300,0.00143) (600,0.000569) (1200,0.000227)
    (2400,9.33e-05) (4800,3.88e-05) (9600,1.6e-05)
};

\addplot[mark=diamond*,color=black] coordinates {
    (150,0.00317) (300,0.000985) (600,0.0003) (1200,8.97e-05)
    (2400,2.64e-05) (4800,7.63e-06) (9600,2.19e-06)
};

\addplot[mark=pentagon*,color=violet] coordinates {
    (150,0.000363) (300,6.42e-05) (600,9.59e-06) (1200,9.61e-07)
};
\end{axis}
\end{tikzpicture}

%% file: alpha=1,25s=2.tex
\begin{tikzpicture}
\begin{axis}[
    width=\linewidth,
    height=0.85\linewidth,
    xlabel={$n$},
    xmode=log,
    ymode=log,
]
\addplot[mark=*,color=blue] coordinates {
    (150,0.00142) (300,0.000629) (600,0.000283) (1200,0.000123)
    (2400,5.59e-05) (4800,2.42e-05) (9600,1.05e-05)
};

\addplot[mark=square*,color=red] coordinates {
    (150,0.000439) (300,0.000118) (600,3.18e-05) (1200,8.28e-06)
    (2400,2.12e-06) (4800,5.16e-07) (9600,1.21e-07)
};

\addplot[mark=triangle*,color=brown] coordinates {
    (150,0.0017) (300,0.000729) (600,0.000317) (1200,0.000134)
    (2400,5.83e-05) (4800,2.5e-05) (9600,1.07e-05)
};

\addplot[mark=diamond*,color=black] coordinates {
    (150,0.000662) (300,0.000176) (600,4.71e-05) (1200,1.23e-05)
    (2400,3.09e-06) (4800,7.42e-07) (9600,1.77e-07)
};

\addplot[mark=pentagon*,color=violet] coordinates {
    (150,7.57e-05) (300,9.93e-06) (600,1.09e-06) (1200,7.51e-08)
};
\end{axis}
\end{tikzpicture}

%% file: alpha=2s=2.tex
\begin{tikzpicture}
\begin{axis}[
    width=\linewidth,
    height=0.85\linewidth,
    xlabel={$n$},
    xmode=log,
    ymode=log
]

\addplot[mark=*,color=blue] coordinates {
    (150,0.000253)
    (300,8.33e-05)
    (600,2.93e-05)
    (1200,9.85e-06)
    (2400,3.60e-06)
    (4800,1.20e-06)
    (9600,4.43e-07)
};

\addplot[mark=square*,color=red] coordinates {
    (150,1.01e-04)
    (300,2.04e-05)
    (600,4.09e-06)
    (1200,8.09e-07)
    (2400,1.58e-07)
    (4800,2.93e-08)
    (9600,6.53e-09)
};

\addplot[mark=triangle*,color=brown] coordinates {
    (150,3.25e-04)
    (300,1.03e-04)
    (600,3.60e-05)
    (1200,1.16e-05)
    (2400,4.10e-06)
    (4800,1.37e-06)
    (9600,5.20e-07)
};

\addplot[mark=diamond*,color=black] coordinates {
    (150,1.35e-04)
    (300,2.73e-05)
    (600,5.54e-06)
    (1200,1.10e-06)
    (2400,2.19e-07)
    (4800,4.17e-08)
    (9600,8.48e-09)
};

\addplot[mark=pentagon*,color=violet] coordinates {
    (150,3.61e-08)
    (300,2.04e-09)
    (600,1.02e-10)
    (1200,3.35e-12)
};
\end{axis}
\end{tikzpicture}

%% file: house_8LExponentialavgTesterror_vs_n.tex
\begin{tikzpicture}
\begin{axis}[
width=\linewidth,
height=0.85\linewidth,
xlabel={$n$},
ylabel={Test error},
legend style={
    at={(0.5,-0.5)},
    anchor=north,
    legend columns=1,
    font=\scriptsize
}
]
\addplot+[mark=*] coordinates {(150,1634307300.1135116) (300,1460133696.3479493) (600,1353126979.4484568) (1200,1272677686.5200672) (2400,1200758381.6354682) };
\addlegendentry{Tail-averaged SGD}
\addplot+[mark=square*] coordinates {(150,1670417504.1682422) (300,1484690996.691344) (600,1365547151.5068338) (1200,1279486113.5488157) (2400,1207198211.635966) };
\addlegendentry{SGDIR}
\addplot+[mark=triangle*] coordinates {(150,1538200238.2246997) (300,1395105853.0767052) (600,1291204080.2957778) (1200,1206040693.2154422) (2400,1131796811.4581115) };
\addlegendentry{KRR}
\end{axis}
\end{tikzpicture}

%% file: elevatorsExponentialavgTesterror_vs_n.tex
\begin{tikzpicture}
\begin{axis}[
width=\linewidth,
height=0.85\linewidth,
xlabel={$n$},
legend style={
    at={(0.5,-0.5)},
    anchor=north,
   legend columns=1,
    font=\scriptsize
}
]
\addplot+[mark=*] coordinates {(150,2.5955766625508665e-05) (300,2.0184411070971206e-05) (600,1.52453618511356e-05) (1200,1.1343850324135988e-05) (2400,9.274074102920346e-06) };
\addlegendentry{Tail-averaged SGD}
\addplot+[mark=square*] coordinates {(150,2.7240534285034384e-05) (300,2.1674426725187944e-05) (600,1.669255957017858e-05) (1200,1.240697770250107e-05) (2400,9.863008328434305e-06) };
\addlegendentry{SGDIR}
\addplot+[mark=triangle*] coordinates {(150,1.977594875963484e-05) (300,1.4627119936186913e-05) (600,1.1347615418125091e-05) (1200,9.316467216605958e-06) (2400,8.110847237040741e-06) };
\addlegendentry{KRR}
\end{axis}
\end{tikzpicture}

%% file: year_prediction_msdExponentialavgTesterror_vs_n.tex
\begin{tikzpicture}
\begin{axis}[
width=\linewidth,
height=0.85\linewidth,
xlabel={$n$},
legend style={
    at={(0.5,-0.5)},
    anchor=north,
    legend columns=1,
    font=\scriptsize
}
]
\addplot+[mark=*] coordinates {(150,108.89238629285748) (300,102.70428533565718) (600,98.4539694775133) (1200,94.006117786641) (2400,91.18586583850168) };
\addlegendentry{Tail-averaged SGD}
\addplot+[mark=square*] coordinates {(150,109.7912542468842) (300,103.672443141404) (600,99.24559045168432) (1200,94.69691379858233) (2400,91.64892077668286) };
\addlegendentry{SGDIR}
\addplot+[mark=triangle*] coordinates {(150,105.12288458194287) (300,99.10554849894847) (600,94.76951460740118) (1200,90.75837552000296) (2400,88.57538463856828) };
\addlegendentry{KRR}
\end{axis}
\end{tikzpicture}